\documentclass[english,envcountsame]{llncs}
\usepackage[T1]{fontenc}
\usepackage[latin9]{inputenc}
\usepackage{float}
\usepackage{textcomp}
\usepackage{amsmath}
\usepackage{amssymb}
\usepackage{stmaryrd}
\usepackage{setspace}
\makeatletter

\newcommand{\lyxmathsym}[1]{\ifmmode\begingroup\def\b@ld{bold}
  \text{\ifx\math@version\b@ld\bfseries\fi#1}\endgroup\else#1\fi}

\newenvironment{lyxlist}[1]
	{\begin{list}{}
		{\settowidth{\labelwidth}{#1}
		 \setlength{\leftmargin}{\labelwidth}
		 \addtolength{\leftmargin}{\labelsep}
		 }}
	{\end{list}}

\usepackage{xcolor}
\definecolor{darkred}{rgb}{0.8,0,0} 
\usepackage{hyperref}
\hypersetup{final, colorlinks=true, linkcolor=black, urlcolor=black, citecolor=black, linkbordercolor=black, pdfborderstyle={/S/U/W 1}}
\usepackage{setspace}
\usepackage[absolute,overlay]{textpos} 
\usepackage{listings}
\usepackage{cite}

\makeatletter
\g@addto@macro \normalsize {%
 \setlength\abovedisplayskip{4pt plus 2pt minus 2pt}%
 \setlength\belowdisplayskip{4pt plus 2pt minus 2pt}%
}
\makeatother

\DeclareRobustCommand{\VAN}[2]{#2}

\usepackage{nameref}
\makeatletter
\let\orgdescriptionlabel\descriptionlabel
\renewcommand*{\descriptionlabel}[1]{%
  \let\orglabel\label
  \let\label\@gobble
  \phantomsection
  \edef\@currentlabel{#1}%
  \let\label\orglabel
  \orgdescriptionlabel{#1}%
}
\makeatother

\makeatletter
\newcommand*{\textlabel}[2]{%
  \edef\@currentlabel{#1}
  \phantomsection
  #1\label{#2}
}
\makeatother

\usepackage{textcomp}
\usepackage{eurosym}

\usepackage{tikz}
\usetikzlibrary{positioning,arrows,calc,fit,automata}
\usepackage{xcolor}

\newcommand{\s}{\scriptstyle}

\newcommand{\cA}[1]{\textcolor{cyan!50!green!85!white}{#1}} 
\newcommand{\dA}[1]{\dot{\textcolor{cyan!50!green!85!white}{#1}}} 

\newcommand{\cB}[1]{\textcolor{violet!85!white}{#1}} 
\newcommand{\dB}[1]{\dot{\textcolor{violet!85!white}{#1}}} 

\newcommand{\cC}[1]{\textcolor{orange}{#1}} 
\newcommand{\dC}[1]{\dot{\textcolor{orange}{#1}}} 

\usepackage{nccmath}
\usepackage{mathtools}

\AtBeginDocument{
  
}

\makeatother

\usepackage{babel}
\begin{document}
\setstretch{1}
\title{Epistemic Planning with Attention \\as a Bounded Resource}
\author{Gaia Belardinelli and Rasmus K. Rendsvig \institute{Center for Information and Bubble Studies, University of Copenhagen \\ \email{\{belardinelli,rasmus\}@hum.ku.dk}}}
\maketitle
\begin{abstract}
Where information grows abundant, attention becomes a scarce resource.
As a result, agents must plan wisely how to allocate their attention
in order to achieve epistemic efficiency. Here, we present a framework
for multi-agent epistemic planning with attention, based on Dynamic
Epistemic Logic (DEL, a powerful formalism for epistemic planning).
We identify the framework as a fragment of standard DEL, and consider
its plan existence problem. While in the general case undecidable,
we show that when attention is required for learning, all instances
of the problem are decidable.
\end{abstract}

\section{Introduction}

The development of autonomous agents is central in artificial intelligence.
A core feature of autonomous agents is their ability to exhibit goal-directed
behaviour, i.e. to commit to goals and generate plans to achieve them.\emph{
Epistemic planning} \cite{Bolander2017} focuses on domains where
agents' plans must take into account their own capabilities and knowledge,
as well as knowledge about other agents' knowledge.\,For example,
in an \emph{epistemic planning problem}, agent $a$ may have as goal
\textquotedblleft $a$ knows the truth-value of $\varphi$, while
$b$ does not know that $a$ knows it\textquotedblright . To achieve
this goal, agent $a$ may need to reason about what it can do to learn
$\varphi$, and about what $b$ may learn about $\varphi$ from $a$\textquoteright s
actions.\,\emph{Dynamic Epistemic Logic} (DEL) offers a highly expressive
basis for epistemic planning, allowing e.g. nondeterminism, partial
observability and arbitrary levels of higher-order reasoning\,\cite{Bolander2017}.

Attention is relevant to autonomous agents that labour in information-rich
environments. As H. Simon wrote in \cite{simon1971}: ``In an information-rich
world, the wealth of information means a dearth of something else:
a scarcity of whatever it is that information consumes. What information
consumes is rather obvious: it consumes the attention of its recipients.
Hence a wealth of information creates a poverty of attention, and
a need to allocate that attention efficiently (...).\textquotedblright{}
Attention is thus a bounded resource, crucial to agents that must
process information to achieve their goals. In such cases, plans must
factor in attention and be attention-feasible: If agents undertake
actions that require more attention than is available, they will fail.
Autonomous agents in multi-agent systems must then turn to epistemic
planning with attention as a bounded resource.

However, DEL does not assume resource bounded agency, and so DEL-based
epistemic planning does not bound the attention agents may consume
to achieve epistemic goals. Therefore, we propose an alteration of
DEL that incorporates attention as a bounded resource in epistemic
planning. 

In the proposed model, attention is used in learning. The model is
of the `DEL-type' by its `state-action-product' format, but with
a twist. \emph{Attention states }portray static snapshots of agents'
information, with agents' remaining attention encoded propositionally.
Attention states are updated with \emph{attention actions,} that have
two elements each: an \emph{action model} interpreted as an \emph{information
source}, and, for each agent, a \emph{question} asked of the information
source (this is the twist). The question, a formula, is what is being
paid attention to. Jointly, action model and questions invoke attention
cost and information change, calculated by taking a state-action product.
Sec.~\ref{sec:Attention} introduces and exemplifies this.

In Sec.~\ref{sec:Emulation-Results}, we relate to DEL in three ways.
We show that attention states may be recast as DEL epistemic states
and \emph{vice versa}; that every DEL action \emph{without} postconditions
may be emulated by an attention action, but not \emph{vice versa};
and that every attention action may be emulated by a DEL action \emph{with}
postconditions, but not \emph{vice versa}. We use these results in
Sec.~\ref{sec:Planning}.

In Sec.~\ref{sec:Planning}, we turn to planning and our main results:
The \emph{plan existence problem} for epistemic planning with attention
is in general undecidable, but is decidable for \emph{No Free Lunch}
(\textsc{nfl}) actions. In \textsc{nfl} actions, complex learning
requires questioning, and all non-trivial questions have non-zero
attention cost. N\textsc{fl} actions are of special interest as they
enforce that attention expenditure is required for active learning.
The first result is a corollary to the undecidability of DEL-based
planning without postconditions \cite{BolanderBirkegaard2011} and
results in Sec.~\ref{sec:Emulation-Results}. The second is a consequence
of bounded attention. When attention is exhausted, the \textsc{nfl}
actions supply a restricted form of background information, ensuring
that eventually no further change occurs. 

Sec.~\ref{sec:Final-Remarks} concludes, discussing how our decidability
result differs from those established in DEL-based epistemic planning,
and future research questions.

\medskip{}

\noindent To the best of the authors' knowledge, the paper's focus
on attention as a bounded resource in epistemic planning is novel.
Some works in the literature relate to different aspects of the present
proposal. \cite{Balbiani2019} proposes a model where explicit beliefs
(interpreted as the focus of agents attention) may be boundedly extended
by perception or inference from memory, using a DEL variant. \cite{Solaki2018}
use DEL to model cognitively costly inference for single agents, using
impossible worlds. Both proposals do not relate to planning. Works
on time-bounded reasoning, where inference requires time, may be interpreted
in terms of attention \cite{Alechina2004}. These models do not use
DEL.
{} Some DEL papers \cite{van2013,Bolander2016} model agents that pay
attention or not, affecting whether they learn, and \cite{Lorini2005}
discusses \emph{joint attentional states. }Neither work captures attention
as a\emph{ resource}. Some works on awareness draw parallels with
attention \cite{Schipper2014}, but without focusing on resource boundedness.
More peripherally lies the field of \emph{attention economics} \cite{hefti2015economics}
which studies attention allocation in markets, but it does not represent
agents' epistemic states, contra DEL.

The proof of some propositions is only sketched in the main text.
Their full version can be found in the Proof Appendix.

\section{\label{sec:Attention}Attention in States and Actions}

Attention may be conceptualized in several ways. It may be equated
with the time, memory or `mental energy' required to learn some
proposition. To accommodate this broad scope, we represent agents'
attention simply as a numerical value with attention costs relative
to agent, question and context (event).

\vspace{-5pt}

\subsection{Language}

Throughout, let \mbox{$I\not=\emptyset$} be a finite set of \emph{agents},
let \mbox{$N\in\mathbb{N}$} be an \emph{attention bound}. We use
a classic epistemic logical language,\,but\,with a bi-partitioned
set of atoms:

Let $\Phi$ be a countable set of \emph{proposition atoms} and let
$\Psi=\{(\alpha_{i}<n),(\alpha_{i}=n)\colon n\leq N,i\in I\}$ be
a set of \emph{attention atoms}. Let the full set of \emph{atoms}
be their disjoint union $At=\Phi\uplus\Psi$. With $i\in I$, $p\in At$,
define the language $\mathcal{L}$ by
\[
\varphi::=\top\mid p\mid\neg\varphi\mid\varphi\wedge\varphi\mid K_{i}\varphi
\]
The formula $K_{i}\varphi$ states that agent $i$ knows that $\varphi$
and the attention atoms $(\alpha_{i}<n)$ and $(\alpha_{i}=n)$ state
that $i$ has, respectively, strictly less than $n$ or exactly $n$,
attention left, where $n$ is a natural number.

As abbreviations, let $(\alpha_{i}>n):=\neg(\alpha_{i}<n)\wedge\neg(\alpha_{i}=n)$
and $(\alpha_{i}\ge n):=\neg(\alpha_{i}<n)$. In figures, we denote
$\neg\varphi$ by $\overline{\varphi}$.

\vspace{-5pt}

\subsection{Attention States}

We use special cases of Kripke models, augmented with maps to quantify
the attention span of each agent at each world (examples are in Sec.~\ref{sec:Attention Actions}).
Sec.~\ref{sec:Emulation-Results} shows them recast as Kripke models
with valuations suitable for attention.
\begin{definition}
\label{def: attention model} An \emph{attention model} is an $M=(W,R,V,A)$
where $W\neq\emptyset$ is a finite set of worlds; $R:I\rightarrow\mathcal{P}(W^{2})$
assigns each $i\in I$ an equivalence relation $R_{i}$; $V:At\rightarrow\mathcal{P}(W)$
is a valuation\emph{; and $A:I\times W\rightarrow\{0,...,N\}$ is
an }attention resource function \emph{satisfying that $wR_{i}v$ implies
$A_{i}(w)=A_{i}(v)$, for all $i\in I$. For $w\in W$, $(M,w)$ is
an }attention state\emph{, with $w$ the }actual world\emph{.}
\end{definition}

\textit{\emph{Relations $R_{i}$ are taken to capture the indistinguishability
of worlds for agent $i$, and are therefore assumed to be equivalence
relations, as standard in epistemic logic \cite{sep-logic-epistemic}.
The restriction on $A_{i}$ ensures that $i$ knows their own attention
span.}}
\begin{definition}
\label{def: semantics attention language}Let $(M,w)=((W,R,V,A),w)$
be an attention state. For $i\in I$, $n\in\mathbb{N}$, $p\in\Phi$,
$(\alpha_{i}<n),(\alpha_{i}=n)\in\Psi$, satisfaction of \textup{$\mathcal{L}$}
formulas is given by

$M,w\vDash p$ \quad{}iff \textup{$w\in V(p)$} for all $p\in\Phi$,

$M,w\vDash(\alpha_{i}=n)$ \quad{}iff $A_{i}(w)=n$,

$M,w\vDash(\alpha_{i}<n)$ \quad{}iff $A_{i}(w)<n$,

$M,w\vDash K_{i}\varphi$ \quad{}iff $wR_{i}v$ implies $M,v\vDash\varphi$

\noindent with standard clauses for $\top$, $\neg\varphi$ and $\varphi\wedge\psi$,
where $\varphi,\psi\in\mathcal{L}$. If for all $(M,w)$, $(M,w)\vDash\varphi\rightarrow\psi$,
we write $\varphi\vDash\psi$.
\end{definition}

Our results rely on establishing equivalence of models, for which
we need to introduce \emph{attention bisimulation}. Attention bisimulation
is defined akin to bisimulation for Kripke models, but factoring attention
into the Atoms clause.
\begin{definition}
\label{def:attention bisimulation} Any two attention states $M=((W,R,V,A),w)$
and $M'=((W',R',V',A'),w')$ are (attention)\emph{ bisimilar} (written
\emph{$(M,w)\leftrightarroweq(M',w')$)} if there exists a relation
$Z\subseteq W\times W'$ such that \textsc{$wZw'$} and for all $v\in W$,
$v'\in W'$, if \textsc{$vZv'$}, then for all $p\in\Phi,i\in I$,

\noindent \textbf{Atoms~} $v\in V(p)\Leftrightarrow v'\in V'(p)$
and $A_{i}(v)=A'_{i}(v')$;

\noindent \textbf{Forth\enskip{}} if $vR_{i}u$, then for some $u'\in W'$,
\textup{$v'R'_{i}u'$} and $uZu'$;

\noindent \textbf{Back \enskip{}} if $v'R'_{i}u'$, then for some
$u\in W$, $vR_{i}u$ and $u'Zu$.
\end{definition}

Bisimulation between attention states implies modal equivalence:
\begin{proposition}
\label{prop:Hennessy-Milner}If attention states $(M,w)$ and $(M',w')$
are bisimilar, then for every $\varphi\in\mathcal{L}$, $(M,w)\vDash\varphi$
iff $(M',w')\vDash\varphi$.
\end{proposition}

\noindent The proof follows Prop.~\ref{prop:K(M)-M-Equivalence }
(Sec.~\ref{sec:Emulation-Results}), of which it makes use.

Finally, in showing our main decidability result, Theorem \ref{thm:NFL-planning-is-decidable},
we use bisimulation contractions and the following Lemma.
\begin{definition}
\label{def:bisim.contraction}Let $(M,w)=((W,R,V,A),w)$ be an attention
model. The \emph{bisimulation contraction} of $(M,w)$ is the attention
model $(M,w)_{\leftrightarroweq}=((W',R',V',A'),[w])$ with $W'=\{[w]\colon w\in W\}$
for $[w]=\{v\in M\colon(M,w)\leftrightarroweq(M,v)\}$; for all $i\in I$,
$R_{i}'=\{([w],[v])\colon\exists w'\in[w],\exists v'\in[v]\text{ with }w'R_{i}v'\}$;
$V':\Phi\rightarrow\mathcal{P}(W)$ with $V'(p)=\{[w]\in W\colon w\in V(p)\}$
for all $p\in\Phi$; and $A'_{i}([w])=A_{i}(w)$, for all $w\in W$.
\end{definition}

\begin{lemma}
\label{lem:bisim.contraction.equivalence}For any attention state
$(M,w)$, for all $\varphi\in\mathcal{L}$, $(M,w)_{\leftrightarroweq}\vDash\varphi$
iff $(M,w)\vDash\varphi$.
\end{lemma}

\begin{proof}
[Proof sketch] $(M,w)_{\leftrightarroweq}$ is bi-similar to $(M,w)$,
witnessed by $Z=\{(w,[w])\colon w\in W\}$. The conclusion then follows
by Prop.~\ref{prop:Hennessy-Milner}.
\end{proof}

\vspace{-10pt}

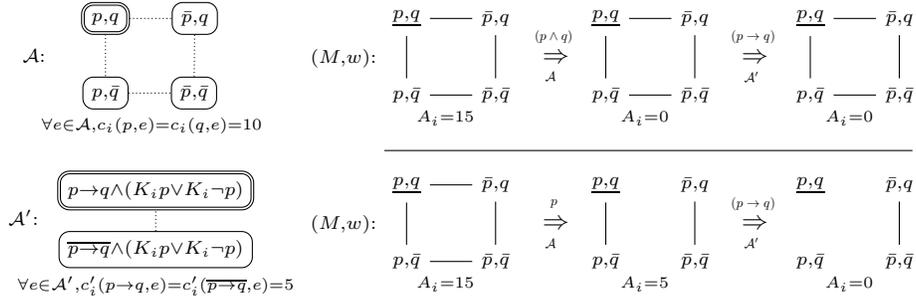
\begin{figure}[h]
\begin{center}
	
	\begin{tikzpicture}\tikzset{event/.style={rectangle,  draw=black, thin, rounded corners, text centered, node distance=16pt}, world/.style={semithick, node distance=16pt}}


True state:
\node [world] (!) {$\s \underline{p,q}$};

\node [world, right=of !](npqB) {$\s \bar{p},q$};
\node [world, below=of !](pnq) {$\s p,\bar{q}$};
\node [world, right=of pnq] (npnq) {$\s \bar{p},\bar{q}$};

name of the model
\node [world, left=of !,  yshift=-15pt , xshift=17pt]  (name) {$\s (M,w):$};

Relations
\draw (!) -- (npqB); 
\draw (!) -- (pnq);
\draw (pnq) -- (npnq); 
\draw  (npnq) -- (npqB); 

Attention span

\node [world, below=of npnq, yshift=20pt, xshift=-19] (A) {\footnotesize$\s A_i=15$};

\node [world, right=of !, xshift=23pt, yshift=-15] (u) {$\Rightarrow$};
\node [world, right=of !, xshift=20pt, yshift=-7] (u) {\scalebox{0.5}[0.5]{$(p\wedge q)$}};
\node [world, right=of !, xshift=24pt, yshift=-22] (u) {\scalebox{0.5}[0.5]{$\mathcal{A}$}};

ATTENTION ACTION MODEL 1, top left

\node [event, accepting, left=of !, xshift= -80pt]  (!) {$\s p,q$};

\node [world, below=of !, yshift=14pt, xshift=-27pt]  (name) {$\s \mathcal{A}:$};

\node [event, right=of !](npq) {$\s \bar{p},q$};
\node [event, below=of !](pnq) {$\s p,\bar{q}$};
\node [event, right=of pnq] (npnq) {$\s \bar{p},\bar{q}$};

Cost
\node [below=of npnq, yshift=29pt, xshift=-16] (A) {\footnotesize$\s \forall e\in\mathcal{A}, c_i(p,e)=c_i(q,e)=10$};

Q*Relations
\draw [densely dotted] (!) -- (npq); 
\draw [densely dotted] (!) -- (pnq);
\draw [densely dotted](pnq) -- (npnq); 
\draw  [densely dotted] (npnq) -- (npq);



ATTENTION ACTION MODEL 2, top right
\node [event, below=of pnq, yshift=-30pt, xshift= 19pt] (!) {$\s \overline{ p\rightarrow q} \wedge  (K_ip\vee K_i\neg p)$};

\node [event, accepting, right=of npq, above=of !, yshift=-8pt](npqA) {$\s p\rightarrow q\wedge (K_ip\vee K_i\neg p)$};

Q*Relations
\draw  [densely dotted] (!) -- (npqA);

Cost
\node [below=of !, yshift=29pt] (A) {\footnotesize$\s \forall e\in\mathcal{A'}, c'_i(p\rightarrow q, e)=c'_i( \overline{ p\rightarrow q},e )=5$};

\node [world, left= of A,  xshift=31pt, yshift=26pt]  (name) {$\s \mathcal{A}':$};

True state:
\node [world, right=of npqB, xshift=8pt] (!) {$\s \underline{p,q}$};

\node [world, right=of !](npq) {$\s \bar{p},q$};
\node [world, below=of !](pnq) {$\s p,\bar{q}$};
\node [world, right=of pnq] (npnq) {$\s \bar{p},\bar{q}$};

Attention span
\node [world, below=of npnq, yshift=20pt, xshift=-19] (A) {\footnotesize$\s A_i=0$};

Relations
\draw (!) -- (npq); 
\draw (!) -- (pnq);
\draw (pnq) -- (npnq); 
\draw  (npnq) -- (npq); 

\node [world, right=of !, xshift=23pt, yshift=-15] (u) {$\Rightarrow$};
\node [world, right=of !, xshift=19pt, yshift=-7] (u) {\scalebox{0.5}[0.5]{$(p\rightarrow q)$}};
\node [world, right=of !, xshift=24pt, yshift=-22] (u) {\scalebox{0.5}[0.5]{$\mathcal{A}'$}};

\node [world, right=of npq, xshift=10pt] (!) {$\s \underline{p,q}$};

\node [world, right=of !](npq) {$\s \bar{p},q$};
\node [world, below=of !](pnq) {$\s p,\bar{q}$};
\node [world, right=of pnq] (npnq) {$\s \bar{p},\bar{q}$};

\node [world, below=of npnq, yshift=20pt, xshift=-19] (A) {\footnotesize$\s A_i=0$};

\draw (!) -- (npq); 
\draw (!) -- (pnq);
\draw (pnq) -- (npnq); 
\draw  (npnq) -- (npq); 


\node [world, yshift=-50, xshift=-12pt](anchor1) {};
\node [world, yshift=-50, xshift=195pt](anchor2) {};
\draw (anchor1) -- (anchor2);


\node [world, yshift=-63pt](!) {$\s \underline{p,q}$};

\node [world, right=of !](npq) {$\s \bar{p},q$};
\node [world, below=of !](pnq) {$\s p,\bar{q}$};
\node [world, right=of pnq] (npnq) {$\s \bar{p},\bar{q}$};

Attention span
\node [world, below=of npnq, yshift=20pt, xshift=-19] (A) {\footnotesize$\s A_i=15$};

name of the model
\node [world, left=of !,  yshift=-15pt , xshift=17pt]  (name) {$\s (M,w):$};

Relations
\draw (!) -- (npq); 
\draw (!) -- (pnq);
\draw (pnq) -- (npnq); 
\draw  (npnq) -- (npq);

\node [world, right=of !, xshift=23pt, yshift=-15] (u) {$\Rightarrow$};
\node [world, right=of !, xshift=26pt, yshift=-7] (u) {\scalebox{0.5}[0.5]{$p$}};
\node [world, right=of !, xshift=24pt, yshift=-22] (u) {\scalebox{0.5}[0.5]{$\mathcal{A}$}};

True state:
\node [world, right=of npq, xshift=8pt] (!) {$\s \underline{p,q}$};

\node [world, right=of !](npq) {$\s \bar{p},q$};
\node [world, below=of !](pnq) {$\s p,\bar{q}$};
\node [world, right=of pnq] (npnq) {$\s \bar{p},\bar{q}$};

Attention span
\node [world, below=of npnq, yshift=20pt, xshift=-19] (A) {\footnotesize$\s A_i=5$};

Relations
\draw (!) -- (pnq);
\draw  (npnq) -- (npq); 

\node [world, right=of !, xshift=23pt, yshift=-15] (u) {$\Rightarrow$};
\node [world, right=of !, xshift=19pt, yshift=-7] (u) {\scalebox{0.5}[0.5]{$(p\rightarrow q)$}};
\node [world, right=of !, xshift=24pt, yshift=-22] (u) {\scalebox{0.5}[0.5]{$\mathcal{A}'$}};

True state:
\node [world, right=of npq, xshift=10pt] (!) {$\s \underline{p,q}$};

\node [world, right=of !](npq) {$\s \bar{p},q$};
\node [world, below=of !](pnq) {$\s p,\bar{q}$};
\node [world, right=of pnq] (npnq) {$\s \bar{p},\bar{q}$};

Attention span
\node [world, below=of npnq, yshift=20pt, xshift=-19] (A) {\footnotesize$\s A_i=0$};

Relation
\draw  (npnq) -- (npq); 

\end{tikzpicture}

\end{center}\caption{\label{fig:Single Agent example} Example for a single agent $i$.
\emph{Left}: attention models $\mathcal{A}$ and $\mathcal{A}'$.
Events are outlined and labeled with their preconditions. Actual events
are thickly outlined. Dotted lines show $Q_{i}^{*}$, with reflexive
links omitted ($Q_{i}$ has only reflexive links). Costs functions
$c_{i},c_{i}'$ are assumed constant across events. Assume $\forall\varphi,\psi\in\mathcal{L},\forall e\in\mathcal{A},\mathcal{A}',c_{i}(\varphi+\psi,e)=c_{i}(\varphi,e)+c_{i}(\psi,e)$,
and the same for $c'$. \emph{Top Right }and \emph{Bottom Right }:
attention state $(M,w)$ and updates. Worlds are labeled with the
literals they satisfy, the actual world is underlined. Full lines
show $R_{i}$. The attention resource $A_{i}$ is constant across
worlds. Arrows show updates and are labeled with the action model
and question used. \textbf{Story: }Agent $i$ is uncertain of the
truth value of theorems $p$ and $q$, but has the goal to learn them.
She can either do proofs ($\mathcal{A}$), or look up known implications
($\mathcal{A}'$). Given her attention, can she succeed? \emph{Top}
\emph{Right}: She attempts to prove $(p\wedge q)$, but drains her
attention and fails to learn anything. \emph{Bottom} \emph{Right}:
She successfully proves $p$, then successfully looks up $p\rightarrow q$.
By this less attention-taxing strategy, she also learns $q$ and so
achieves her goal.}
\vspace{-14pt}
\end{figure}
\subsection{\label{sec:Attention Actions}Attention Actions}

To perform actions on attention states, we use an augmented version
of action models (without postconditions), well-known from DEL \cite{BaltagBMS_1998},
plus a map capturing what question (formula) each agent pays attention
to. 

\begin{definition}
\label{def: attention action model and action model} An \emph{attention
action}\textbf{\emph{ }}\emph{model} is a tuple $\mathcal{A}=(E,Q,Q^{*},pre,c)$
where $E\neq\emptyset$ is a finite set of events; $Q:I\rightarrow\mathcal{P}(E^{2})$
and $Q^{*}:I\rightarrow\mathcal{P}(E^{2})$ assign each $i\in I$
equivalence relations $Q_{i}$ and $Q_{i}^{*}$; $\mbox{\ensuremath{pre:E\rightarrow\mathcal{L}}}$
assigns a precondition to each event; $c:I\times\mathcal{L}\times E\rightarrow\mathbb{N}$
is a \emph{cost function} satisfying that $c_{i}(\top,e)=0$ and that
$(e,e')\mbox{\ensuremath{\in Q_{i}\cup Q_{i}^{*}}}$ implies $c_{i}(\varphi,e)=c_{i}(\varphi,e')$,
for all $i\in I$, $\varphi\in\mathcal{L}$.

A \emph{question map} $\text{€}:I\rightarrow\mathcal{L}$ assigns
each agent a formula.

For $e\in E$, a triple $(\mathcal{A},\text{€},e)$ is an \emph{attention
action}, with $e$ the \emph{actual event}. Let \textsc{all} be the
class of attention actions.
\end{definition}
The interpretation of attention actions is largely on par with the
interpretation of epistemic actions from standard DEL. Preconditions
state the conditions under which events can occur, and the $Q_{i}$
relation represents \emph{unavoidable} indistinguishability: If $(e,e')\in Q_{i}$,
then $i$ simply cannot distinguish $e$ and $e'$. However, the interpretation
of $Q^{*}$ is novel: If $(e,e')\in Q_{i}^{*}$, then the events are
by \emph{default }indistinguishable for $i$, but may be distinguishable
if $i$ pays attention to the right question (and has the attention
to do so). Through the product defined below, $Q_{i}^{*}$ interacts
with $i$'s question $\text{€}_{i}$, its cost, and $i$'s attention
span: 
{} If $i$ asks a question which $e$ and $e'$ answer differently (i.e.
$pre(e)\vDash\text{€}_{i}$ and $pre(e')\not\vDash\text{€}_{i}$)
and for which $i$ has sufficient attention ($A_{i}(w)\geq c_{i}(\text{€}_{i},e)$),
then $e$ and $e'$ will be distinguishable for $i$ in world $w$.
This reduces $i$'s attention in the updated model. Hence agents may
learn the answers to their questions by paying attention.

About costs, we only assume that agents know the cost of each formula,
and that asking no question ($\top$) is free. One could also require
e.g. that $c_{i}(K_{j}\varphi,e)\geq c_{i}(\varphi,e)$ or $c_{i}(\varphi\wedge\psi,e)\ge c_{i}(\varphi,e)+c_{i}(\psi,e)$.
Such assumptions would not influence the paper's results. About questions,
we make the strong assumption that what is payed attention to is common
knowledge. In Sec.\ref{sec:Final-Remarks}, we remark on lifting
this.
\begin{definition}
\label{def: Attention Update} Let $(M,w)=((W,R,V,A),w)$ be an attention
state and let $X=(\mathcal{A},\text{€},e)$ be an attention action
with $\mathcal{A}=(E,Q,Q^{*},pre,c)$. For all $i\in I$ and $\varphi\in\mathcal{L}$,
let 
\[
Q_{i}^{*}[\varphi|\neg\varphi]=Q_{i}^{*}\setminus\{(e,e')\in Q_{i}^{*}\colon pre(e)\vDash\varphi\text{ and }pre(e')\not\vDash\varphi\}.
\]
 The attention update of $(M,w)$ with $X$ is $(M,w)\otimes X=((W^{X},R^{X},V^{X},A^{X}),(w,e))$
where
\begin{lyxlist}{00.00.0000}
\item [{\textup{$W^{X}=\{(w,e)\in W\times E\colon M,w\vDash pre(e)\}$,}}]~
\item [{$R^{X}$~is~given~by~$((w,e),(v,f))\in R_{i}^{X}$~iff~$(w,v)\in R_{i}$~and}]~
\item [{~$(e,f)\in\begin{cases}
Q_{i}\cup Q_{i}^{*} & \text{if }c_{i}(\lyxmathsym{€}_{i},e)>A_{i}(w)\\
Q_{i}\cup Q_{i}^{*}\mbox{[\ensuremath{\lyxmathsym{€}_{i}\mid\neg\lyxmathsym{€}_{i}}]} & \text{if }c_{i}(\lyxmathsym{€}_{i},e)\le A_{i}(w)\not=0
\end{cases}$}]~
\item [{\textup{$V^{X}(p)=\{(w,e)\in W^{X}\colon w\in V(p)\}$}}] \textup{for
all $p\in\Phi$,}
\item [{$A_{i}^{X}(w,e)=\max\{0,A{}_{i}(w)-c_{i}(\lyxmathsym{€}_{i},e)\}$,}]~
\end{lyxlist}
Call~$X$~\emph{applicable}~to~$(M,w)$~if~$M,w\vDash pre(e)$,~else~not.
Where $\sigma$ is a (potentially infinite) sequence of attention
actions, if for all $k\leq n$, $X_{k+1}$ is applicable to $(M,w)\otimes X_{1}\otimes\cdots\otimes X_{k}$,
denote $(M,w)\otimes X_{1}\otimes\cdots\otimes X_{n}$ by $(M,w)^{\sigma}=((W^{\sigma},R^{\sigma},V^{\sigma},A^{\sigma}),w^{\sigma})$
and call $\sigma$ \emph{applicable} to $(M,w)$.
\end{definition}

\vspace{-7pt}
\subsection{No Free Lunch}

The attention actions of Figures~\ref{fig:Single Agent example}
and~\ref{fig:MuddyChildren} belong to a class shown special interest
in this paper: the class of \emph{No Free Lunch }(\textsc{nfl}) actions.
\begin{definition}
Let \textsc{nfl} be the class of attention actions $(\mathcal{A},\text{€},e)=((E,Q,Q^{*},pre),\lyxmathsym{€},e)$
that satisfy, for all $i\in I$, $Q_{i}^{*}=E\times E$ and for all
$e\in E$, $\varphi\in\mathcal{L}\backslash\{\top\}$, $c_{i}(\varphi,e)>0$.
\end{definition}

We find \textsc{nfl} actions to be of special interest as they enforce
attention use for all non-trivial questions, and thus respect that
attention is a bounded resource for learning. Jointly, the two restrictions
entail that any change in information comes at some cost to attention,
with the exception that agents always learn an \textsc{nfl} action's
unavoidable ``background information'' $\bigvee_{e\in E}pre(e)$.
Special cases where $\top\vDash\bigvee_{e\in E}pre(e)$ then ensure
no learning without attention cost, as in $\mathcal{A},\mathcal{A}'$
of Figure~\ref{fig:Single Agent example} and$\mathcal{A},\mathcal{A}''$
of Figure~\ref{fig:MuddyChildren}. Stated differently, then an \textsc{nfl}
action applied without attention cost is equivalent to the \emph{public
announcement} of its background information. This is the content of
Lemma~\ref{lem:eq.to.pub.ann} below, used later to show decidability.
\begin{definition}
\label{def:background.announcement}For \textup{\negthinspace{}}any\textup{\negthinspace{}}
$X\!\!=\!\!((E,\!Q,\!Q^{*},pre,c),\!\text{€},\!e)$,~its \emph{background
announcement} \textup{is $X!\!\!=\!\!((E^{!},\!Q^{!},\!Q^{*!},\!pre^{!},\!c^{!}),\!\text{€}^{!},\!e^{!})$
for $E^{!}\!=\!\{e^{!}\}$,\negthinspace{} $Q^{!}{}_{i}\!=\!Q_{i}^{*!}\!=\!\{(e^{!},e^{!})\}$,
$pre^{!}(e^{!})\!=\!\bigvee_{e\in E}pre(e)$, $c^{!}=c$ and $\lyxmathsym{€}_{i}^{!}=\top$,
for all $i\in I$.}
\end{definition}

\begin{lemma}
\label{lem:eq.to.pub.ann}For\textup{\negthinspace{}} any\textup{\negthinspace{}}
$X\!\in\!\textsc{nfl}$, for\textup{\negthinspace{}} any attention
state $(M,w)\!=\!((W,R,V,A),w)$, if $A^{X}\!=\!A$, then $(M,w)^{X}\!\leftrightarroweq\!(M,w)^{X!}$.
\end{lemma}

\begin{proof}
Let $X=((E,Q,Q^{*},pre,c),\text{€},e)$. Then $Z=\{((w,e),(w,e^{!})):w\in W,e\in E\}\subseteq(M,w)^{X}\times(M,w)^{X!}$
is a bisimulation: \textbf{Atoms:} for all $p\in\Phi$, $w\in W$,
$e\in E$, $i\in I$, $(w,e)\in V^{X}(p)$ iff $(w,e^{!})\in V^{X!}(p)$
and $A_{i}^{X}(w,e)=A_{i}^{X!}(w,e^{!})$ as both $A^{X}=A$ and $A^{X!}=A$,
the latter as $\lyxmathsym{€}_{i}^{!}=\top$ for all $i\in I$. \textbf{Forth:}
Assume $(w,e)Z(w,e^{!})$ and $(w,e)R_{i}^{X}(v,f)$ for some $i\in I$.
Then $wR_{i}v$ and as $Q_{j}^{!}=Q_{j}^{*!}=\{(e^{!},e^{!})\}$ for
all $j\in I$, so $(w,e^{!})R{}_{i}^{X!}(v,e^{!})$. Hence $(v,f)Z(v,e^{!})$.
\textbf{Back}: Assume $(w,e)Z(w,e^{!})$ and $(w,e^{!})R_{i}^{X!}(v,e^{!})$
for some $i\in I$. Then $wR_{i}v$. As $(v,e^{!})\in(M,w)^{X!}$,
$M,v\vDash pre^{!}(e^{!})$, i.e. $M,v\vDash\bigvee_{e\in E}pre(e)$
(by Defs.~\ref{def:background.announcement} and \ref{def: Attention Update}).
Hence, for some $f\in E$, $M,v\vDash pre(f)$. So $(v,f)\in(M,w)^{X}$.
As $X\in\text{\textsc{nfl}}$, $eQ^{*}f$. Hence, also $(w,e)R_{i}^{X}(v,f)$
(cf. Def.~\ref{def: Attention Update}) as $A^{X}=A$ implies that
either $\text{€}_{i}=\top$ or $A_{i}(w)=0<c_{i}(\lyxmathsym{€}_{i},e)$.
Finally, $(v,f)Z(v,e^{!})$. Hence, $(M,w)^{X}\leftrightarroweq(M,w)^{X!}$.
\end{proof}
\vspace{-1pt}
The requirement $A^{X}=A$ of Lemma~\ref{lem:eq.to.pub.ann} states that
no agent spends attention. We remark that if agent $i$ spends attention
and learns, then, even if $j$ does not spend attention, $j$'s higher-order
will still change. This is as in standard DEL where any informational
change one way or the other affects all agents.
\begin{figure}[H]
\begin{tikzpicture}\tikzset{event/.style={rectangle, draw=black, thin, rounded corners, text centered, node distance=16pt, minimum height=12pt}, world/.style={thick, node distance=16pt}}



\node[event, accepting, yshift=-10pt, xshift=-50pt] (A) {$\s  \varphi_C $ };

\node [below=of A, yshift=28pt] (c) {\footnotesize$\s\forall \varphi\in \mathcal{L}\setminus \{\top\}, c_i(\varphi,e)=1$};

\node [world, left=of A,  xshift=12pt ]  (name) {$\s \mathcal{A}:$};


\node[event, accepting, below=of A, yshift=-39] (B) {$\s  \varphi_C $ };
\node[event, above=of B, yshift=-8pt] (B') {$\s \overline{\varphi_C}$ };

\path (B) edge (B')[-, semithick, densely dotted] node[right, xshift=-1pt, yshift=10pt] {\scalebox{0.5}[0.5]{$a,b,c$}}   (B');

\node [world, left=of B,  xshift=17pt ]  (name) {$\s \mathcal{A'}:$};

\node [below=of B , yshift=28pt] (c) {\footnotesize$\s \forall \varphi\in\mathcal{L}\setminus \{\top\}, \forall e\in\mathcal{A'},$};
\node [below=of B , yshift=20pt] (c) {\footnotesize$\s  c_i(\varphi,e)=1$};


\node[event, accepting, below=of B, yshift=-40] (B) {$\s  \varphi_C $ };
\node[event, above=of B, yshift=-8pt] (B') {$\s \overline{\varphi_C}$ };

\path (B) edge [-, bend right,  semithick] node [xshift=3]{\scalebox{0.5}[0.5]{$c$}} (B');

\path (B) edge [-, bend left,  semithick, densely dotted] node [xshift=-6]{\scalebox{0.5}[0.5]{$a,b$} } (B');

\node [world, left=of B,  xshift=17pt ]  (name) {$\s \mathcal{A''}:$};

\node [below=of B , yshift=28pt] (c) {\footnotesize$\s \forall \varphi\in \mathcal{L}\setminus \{\top\}, \forall e\in\mathcal{A''},$};
\node [below=of B , yshift=20pt] (c) {\footnotesize$\s  c_i(\varphi,e)=1$};

%

\node [world] (!) {$\s \underline{\dA{1}\dB{2}\dC{2}}$};


\node [world, above=of !, xshift=20pt, yshift=-10pt](ddc) {$\s \dA{1}\dB{2}\cC{2}$};
\node [world, right=of !](cdd) {$\s \cA{1}\dB{2}\dC{2}$};
\node [world, below=of !](dcd) {$\s \dA{1}\cB{2}\dC{2}$};

\node [world, below=of ddc] (dcc) {$\s \dA{1}\cB{2}\cC{2}$};
\node [world, right=of ddc] (cdc) {$\s \cA{1}\dB{2}\cC{2}$};
\node [world, right=of dcd] (ccd) {$\s \cA{1}\cB{2}\dC{2}$};

\draw [color=cyan!50!green!85!white] (dcd) -- (ccd); 
\draw [color=cyan!50!green!85!white] (ddc) -- (cdc); 

\draw [color=violet!85!white] (ddc) -- (dcc); 
\draw [color=violet!85!white] (cdd) -- (ccd); 

\draw [color=orange] (cdd) -- (cdc); 
\draw [color=orange] (dcd) -- (dcc); 

\draw [color=cyan!50!green!85!white] (!) -- (cdd); 
\draw [color=violet!85!white] (!) -- (dcd); 
\draw [color=orange] (!) -- (ddc); 


\node [world, right=of !, xshift=26pt] (u) {$\Rightarrow$};
\node [world, above=of u, xshift=-1pt, yshift=-21] (t) {\scalebox{0.5}[0.5]{$\top$}};
\node [world, below=of u, xshift=-2pt, yshift=22] (v) {\scalebox{0.5}[0.5]{$\mathcal{A}$}};

\node [world, right=of u, xshift=-15pt] (!) {$\s \underline{\dA{1}\dB{2}\dC{2}}$};

\node [world, above=of !, xshift=20pt, yshift=-10pt](ddc) {$\s \dA{1}\dB{2}\cC{2}$};
\node [world, right=of !](cdd) {$\s \cA{1}\dB{2}\dC{2}$};
\node [world, below=of !](dcd) {$\s \dA{1}\cB{2}\dC{2}$};

\draw [color=cyan!50!green!85!white] (!) -- (cdd); 
\draw [color=violet!85!white] (!) -- (dcd); 
\draw [color=orange] (!) -- (ddc); 


\node [world, right=of !, xshift=26pt] (u) {$\Rightarrow$};
\node [world, above=of u, xshift=-1pt, yshift=-21] (t) {\scalebox{0.5}[0.5]{$\top$}};
\node [world, below=of u, xshift=-2pt, yshift=22] (v) {\scalebox{0.5}[0.5]{$\mathcal{A}$}};

\node [world, right=of u, xshift=-15pt] (!) {$\s \underline{\dA{1}\dB{2}\dC{2}}$};


\node [world, yshift=-38, xshift=-15pt](anchor1) {};
\node [world, yshift=-38, xshift=225pt](anchor2) {};
\draw (anchor1) -- (anchor2);


\node [world, yshift=-68pt](!) {$\s \underline{\dA{1}\dB{2}\dC{2}}$};


\node [world, above=of !, xshift=20pt, yshift=-10pt](ddc) {$\s \dA{1}\dB{2}\cC{2}$};
\node [world, right=of !](cdd) {$\s \cA{1}\dB{2}\dC{2}$};
\node [world, below=of !](dcd) {$\s \dA{1}\cB{2}\dC{2}$};

\node [world, below=of ddc] (dcc) {$\s \dA{1}\cB{2}\cC{2}$};
\node [world, right=of ddc] (cdc) {$\s \cA{1}\dB{2}\cC{2}$};
\node [world, right=of dcd] (ccd) {$\s \cA{1}\cB{2}\dC{2}$};

\draw [color=cyan!50!green!85!white] (dcd) -- (ccd); 
\draw [color=cyan!50!green!85!white] (ddc) -- (cdc); 

\draw [color=violet!85!white] (ddc) -- (dcc); 
\draw [color=violet!85!white] (cdd) -- (ccd); 

\draw [color=orange] (cdd) -- (cdc); 
\draw [color=orange] (dcd) -- (dcc); 

\draw [color=cyan!50!green!85!white] (!) -- (cdd); 
\draw [color=violet!85!white] (!) -- (dcd); 
\draw [color=orange] (!) -- (ddc); 

\node [world, right=of !, xshift=26pt] (u) {$\Rightarrow$};
\node [world, above=of u, xshift=-1pt, yshift=-21] (t) {\scalebox{0.5}[0.5]{$\varphi_C$}};
\node [world, below=of u, xshift=-2pt, yshift=22] (v) {\scalebox{0.5}[0.5]{$\mathcal{A'}$}};

\node [world, right=of u, xshift=-15pt] (!) {$\s \underline{\dA{0}\dB{1}\dC{1}}$};

\node [world, above=of !, xshift=20pt, yshift=-10pt](ddc) {$\s \dA{0}\dB{1}\cC{1}$};
\node [world, right=of !](cdd) {$\s \cA{0}\dB{1}\dC{1}$};
\node [world, below=of !](dcd) {$\s \dA{0}\cB{1}\dC{1}$};

\node [world, below=of ddc] (dcc) {$\s \dA{0}\cB{1}\cC{1}$};
\node [world, right=of ddc] (cdc) {$\s \cA{0}\dB{1}\cC{1}$};
\node [world, right=of dcd] (ccd) {$\s \cA{0}\cB{1}\dC{1}$};

\draw [color=cyan!50!green!85!white] (!) -- (cdd); 
\draw [color=violet!85!white] (!) -- (dcd); 
\draw [color=orange] (!) -- (ddc); 

\node [world, right=of !, xshift=26pt] (u) {$\Rightarrow$};
\node [world, above=of u, xshift=-1pt, yshift=-21] (t) {\scalebox{0.5}[0.5]{$\varphi_C$}};
\node [world, below=of u, xshift=-2pt, yshift=22] (v) {\scalebox{0.5}[0.5]{$\mathcal{A'}$}};

\node [world, right=of u, xshift=-15pt] (!) {$\s \underline{\dA{0}\dB{0}\dC{0}}$};

\node [world, above=of !, xshift=20pt, yshift=-10pt](ddc) {$\s \dA{0}\dB{0}\cC{0}$};
\node [world, right=of !](cdd) {$\s \cA{0}\dB{0}\dC{0}$};
\node [world, below=of !](dcd) {$\s \dA{0}\cB{0}\dC{0}$};

\node [world, below=of ddc] (dcc) {$\s \dA{0}\cB{0}\cC{0}$};
\node [world, right=of ddc] (cdc) {$\s  \cA{0}\dB{0}\cC{0}$};
\node [world, right=of dcd] (ccd) {$\s \cA{0}\cB{0}\dC{0}$};

\draw [color=cyan!50!green!85!white] (!) -- (cdd); 


\node [world, yshift=-108, xshift=-16pt](anchor1) {};
\node [world, yshift=-108, xshift=225pt](anchor2) {};
\draw (anchor1) -- (anchor2);


\node [world, yshift=-138pt](!) {$\s \underline{\dA{0}\dB{2}\dC{2}}$};


\node [world, above=of !, xshift=20pt, yshift=-10pt](ddc) {$\s \dA{0}\dB{2}\cC{2}$};
\node [world, right=of !](cdd) {$\s \cA{0}\dB{2}\dC{2}$};
\node [world, below=of !](dcd) {$\s \dA{0}\cB{2}\dC{2}$};

\node [world, below=of ddc] (dcc) {$\s \dA{0}\cB{2}\cC{2}$};
\node [world, right=of ddc] (cdc) {$\s \cA{0}\dB{2}\cC{2}$};
\node [world, right=of dcd] (ccd) {$\s \cA{0}\cB{2}\dC{2}$};

\draw [color=cyan!50!green!85!white] (dcd) -- (ccd); 
\draw [color=cyan!50!green!85!white] (ddc) -- (cdc); 

\draw [color=violet!85!white] (ddc) -- (dcc); 
\draw [color=violet!85!white] (cdd) -- (ccd); 

\draw [color=orange] (cdd) -- (cdc); 
\draw [color=orange] (dcd) -- (dcc); 

\draw [color=cyan!50!green!85!white] (!) -- (cdd); 
\draw [color=violet!85!white] (!) -- (dcd); 
\draw [color=orange] (!) -- (ddc); 

\node [world, right=of !, xshift=26pt] (u) {$\Rightarrow$};
\node [world, above=of u, xshift=-1pt, yshift=-21] (t) {\scalebox{0.5}[0.5]{$\varphi_C$}};
\node [world, below=of u, xshift=-2pt, yshift=22] (v) {\scalebox{0.5}[0.5]{$\mathcal{A''}$}};

\node [world, right=of u, xshift=-15pt] (!) {$\s \underline{\dA{0}\dB{1}\dC{1}}$};

\node [world, above=of !, xshift=20pt, yshift=-10pt](ddc) {$\s \dA{0}\dB{1}\cC{1}$};
\node [world, right=of !](cdd) {$\s \cA{0}\dB{1}\dC{1}$};
\node [world, below=of !](dcd) {$\s \dA{0}\cB{1}\dC{1}$};

\node [world, below=of ddc] (dcc) {$\s \dA{0}\cB{1}\cC{1}$};
\node [world, right=of ddc] (cdc) {$\s \cA{0}\dB{1}\cC{1}$};
\node [world, right=of dcd] (ccd) {$\s \cA{0}\cB{1}\dC{1}$};

\draw [color=cyan!50!green!85!white] (dcd) -- (ccd); 
\draw [color=cyan!50!green!85!white] (ddc) -- (cdc); 
\draw [color=orange] (cdd) -- (cdc); 
\draw [color=orange] (dcd) -- (dcc); 

\draw [color=cyan!50!green!85!white] (!) -- (cdd); 
\draw [color=violet!85!white] (!) -- (dcd); 
\draw [color=orange] (!) -- (ddc); 

\node [world, right=of !, xshift=26pt] (u) {$\Rightarrow$};
\node [world, above=of u, xshift=-1pt, yshift=-21] (t) {\scalebox{0.5}[0.5]{$\varphi_C$}};
\node [world, below=of u, xshift=-2pt, yshift=22] (v) {\scalebox{0.5}[0.5]{$\mathcal{A''}$}};

\node [world, right=of u, xshift=-15pt] (!) {$\s \underline{\dA{0}\dB{0}\dC{0}}$};

\node [world, above=of !, xshift=20pt, yshift=-10pt](ddc) {$\s \dA{0}\dB{0}\cC{0}$};
\node [world, right=of !](cdd) {$\s \cA{0}\dB{0}\dC{0}$};
\node [world, below=of !](dcd) {$\s \dA{0}\cB{0}\dC{0}$};

\node [world, below=of ddc] (dcc) {$\s \dA{0}\cB{0}\cC{0}$};
\node [world, right=of ddc] (cdc) {$\s \cA{0}\dB{0}\cC{0}$};
\node [world, right=of dcd] (ccd) {$\s \cA{0}\cB{0}\dC{0}$};

\draw [color=cyan!50!green!85!white] (dcd) -- (ccd); 
\draw [color=cyan!50!green!85!white] (ddc) -- (cdc); 
\draw [color=orange] (cdd) -- (cdc); 
\draw [color=orange] (dcd) -- (dcc); 

\draw [color=cyan!50!green!85!white] (!) -- (cdd); 
\draw [color=orange] (!) -- (ddc); 
\draw [color=violet!85!white] (!) -- (dcd); 
\end{tikzpicture}
\caption{\label{fig:MuddyChildren}
Three variants of Muddy Children for $I=\{a,b,c\}$.Let $\varphi_{C}$ be short for $\bigwedge_{i\in I}\neg K_{i}d\wedge\neg K_{i}\neg d$:
No child knows whether they are dirty ($d$). \emph{On the left column}:
Three attention action models. $\mathcal{A}$ is a public announcement
of $\varphi_{C}$, the update used in the standard modeling of Muddy
Children. In $\mathcal{A}'$, the children may fail to pay attention
to the announcement. $\mathcal{A}''$ is as $\mathcal{A}'$, except
that $c$ cannot hear the announcement (the full line is the $Q$-relation).
\emph{Rows 1, 2, and 3} show different dynamics, each starting \emph{after}
the parent's announcement that at least one child is dirty. Each world
$w$ is labeled $A_{a}(w)A_{b}(w)A_{c}(w)$, with $A_{i}(w)$ dotted
if $i$ is dirty in $w$. Horizontal lines show $R_{a}$, vertical
lines $R_{b}$, and diagonal lines $R_{c}$. The arrows show updates
and are labeled with the action model and the question used by all
agents. \emph{Row 1}: Standard Muddy Children (see e.g. \cite{Fagin_etal_1995}):
The children pay only attention to $\top$, but are forced to learn.
\emph{Row }2: Muddy Children with limited attention: The children
all pay attention to $\varphi_{C}$. After the first announcement,
$a$ runs out of attention, so does not learn its own state by the
second announcement. \emph{Row }3: As Row 2, but $a$ starts with
zero attention and $c$ cannot hear the announcement, despite paying
attention. Nobody learns their state.}
\end{figure}
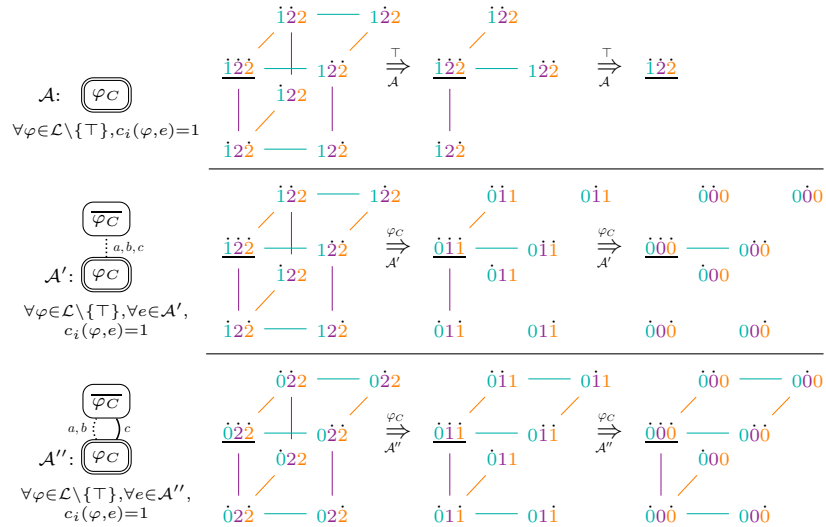
\section{\label{sec:Emulation-Results}Emulation Results}

This section shows that attention states may be recast as DEL epistemic
states and \emph{vice versa}, that every DEL epistemic action \emph{without}
postconditions may be emulated by an attention action but not \emph{vice
versa}, and that every attention action may be emulated by an epistemic
action \emph{with} postconditions, but not \emph{vice versa}. These
results facilitate proving the above modal equivalence result (Prop.
\ref{prop:Hennessy-Milner}) and undecidability of general attention
planning (Thm.~\ref{thm:All-is-undecidable}) as corollaries to existing
results. They further show the relationship with DEL and allow us
to place our results on epistemic planning with attention in the wider
context of DEL-based epistemic planning in Sec.~\ref{sec:Final-Remarks}.

\subsection{Attention States and Epistemic States}
\begin{definition}
\label{def:K-transform}\label{def: semantics for epistemic states}An
\emph{epistemic Kripke model}\textbf{\emph{ }}for $I$ and $At$ is
a tuple $K=(W,R,V)$ where $W\neq\emptyset$ is a finite set of worlds,
$R:I\rightarrow\mathcal{P}(W^{2})$ assigns each $i\in I$ an equivalence
relation $R_{i}$; $V:At\rightarrow\mathcal{P}(W)$ is a valuation.
For $w\in W$, $(K,w)$ is called an epistemic state.

Truth of \textup{$\mathcal{L}$}-formulas over $(K,w)$ is as in
Def.\,\ref{def: semantics attention language}, except for the atomic
clause:
\begin{lyxlist}{00.00.0000}
\item [{$K,w\vDash p$}] iff $w\in V(p)$ for all $p\in\Phi\cup\Psi$.
\end{lyxlist}
The \emph{Kripke rendition }of the attention state $(M,w)=((W,R,V,A),w)$
is $K(M,w)=((W,R,V_{A}),w)$ where for all $p\in\Phi\cup\Psi$, \mbox{$V_{A}(p)=\{w\in W:M,w\vDash p\}$}.
\end{definition}

Our first emulation result shows that the class of attention states
can be embedded in the class of epistemic states:
\begin{proposition}
\label{prop:K(M)-M-Equivalence }For any attention state $(M,w)$,
$K(M,w)$ is $i)$ an epistemic state; $ii)$ isomorphic to $(M,w)$;
$iii)$ for all $\varphi\in\mathcal{L}$, $K(M,w)\vDash\varphi$ iff
$(M,w)\vDash\varphi$. 
\end{proposition}

\begin{proof}
Immediate. See Proof Appendix for details.
\end{proof}

Prop.~\ref{prop:K(M)-M-Equivalence } allows us to establish Prop.~\ref{prop:Hennessy-Milner}
as a corollary. For standard definitions of a bisimulation for Kripke
models and the corresponding Hennessy-Milner Theorem, see \cite[Def. 2.16 and Thm. 2.24]{BlueModalLogic2001}.
\begin{proof}
[Proof of Prop. \ref{prop:Hennessy-Milner}] Assume that $(M,w)$
and $(M',w')$ are bisimilar witnessed by $Z$. Then $Z$ is also
a (Kripke model) bisimulation between $K(M,w)$ and $K(M',w')$, as
is evident from their definition. By Prop.~\ref{prop:K(M)-M-Equivalence },
if $(M,w)\vDash\varphi$, then also $K(M,w)\vDash\varphi$, and by
the Hennessy-Milner Theorem, $K(M',w')\vDash\varphi$. Again by Prop.~\ref{prop:K(M)-M-Equivalence },
$(M',w')\vDash\varphi$.
\end{proof}

Our second emulation result inverts the first, showing that the class
of epistemic states can be embedded in the class of attention states.
As attention states have specific conditions for attention atom satisfaction,
equi-satisfaction holds only for the sublanguage $\mathcal{L}_{\Phi}\subseteq\mathcal{L}$
of formulas without attention atoms. As $\Phi$ can be arbitrarily
extended, we ignore this point when we use the result in showing undecidability
(Thm.~\ref{thm:All-is-undecidable}).
\begin{proposition}
\label{prop:static-attention-emulates-Kripke}For every epistemic
state $(K,w)$, there is an attention state $(M,v)$ such that for
all $\varphi\in\mathcal{L}_{\Phi}$, $(K,w)\vDash\varphi$ iff $(M,v)\vDash\varphi$
.
\end{proposition}

\begin{proof}
Adding to $K$ any attention resource function $A$ yields an attention
model $M$ for which $(M,w)$ is as desired.
\end{proof}

\subsection{Attention Actions and Epistemic Actions}
\begin{definition}
\label{def: action models} An \emph{action model} is a tuple $\mathcal{E}=(E,Q,pre,post)$
with $E\neq\emptyset$ a finite set of events; $Q\!:\!I\rightarrow\mathcal{P}(E^{2})$
assigns each $i\in I$ an equivalence relation $Q_{i}$;\linebreak
\mbox{$pre\!:\!E\rightarrow\mathcal{L}$}, \mbox{$post\!:\!E\!\rightarrow\!(At\!\rightarrow\!\mathcal{L})$}
assign each action a pre- and a post-condition, respectively. $\mathcal{E}$
is \emph{without postconditions} if for all $e\in E$, $post(e)(p)\!=\!p$.

For $e\in E$, $(\mathcal{E},e)$ is an \emph{epistemic action}, with
$e$ the \emph{actual event}. Let $\textsc{Post}$ and $\textsc{noPost}$
denote, respectively, the classes of all epistemic actions all epistemic
actions without postconditions.
\end{definition}

\begin{definition}
\label{def: standard product updates}Let $(K,w)=((W,R,V),w)$ be
an epistemic state for $At$ and $X=((E,Q,pre,post),e)$ an epistemic
action. The \emph{product update }of $(K,w)$ with $X=(\mathcal{E},e)$
is $(K,w)^{X}=((W^{X},R^{X},V^{X}),(w,e))$, where $W^{X}=\{(w,e)\in W\times E\colon M,w\vDash pre(e)\}$,
$((w,e),(v,f))\in R_{i}^{X}$ iff $(w,v)\in R_{i}$ and $(e,f)\in Q$
for all $i\in I$, and\textup{ $V^{X}(p)=\{(w,e)\in W^{X}:M,w\vDash post(e)(p)\}$.
}Definitions of, and notation for, \emph{applicability} is parallel
to Def.~\ref{def: Attention Update}.
\end{definition}

\begin{definition}
\label{def: equivalence of action models}For $Y=(\mathcal{A},\text{€},e')\in\textsc{all}$,
$X=(\mathcal{E},e)\in\textsc{Post}\cup\textsc{noPost}$, $Y$ is \emph{equivalent}
to $X$ iff for all attention states $(M,w)$, all $\varphi\in\mathcal{L}$,
$(M,w)^{Y}\vDash\varphi$ iff $(K(M,w))^{X}\vDash\varphi$.
\end{definition}

\begin{proposition}
\label{prop:AttUp-emulates-NoPost}For all $X\in\textsc{noPost}$,
there is an equivalent $Y\in\textsc{all}$.
\end{proposition}

\begin{proof}
[Proof sketch] Let $X=((E,Q,pre,post),e)\in\textsc{noPost}$. To
emulate $X$, we build an attention action $Y=((E',Q',Q^{*},pre',c),\text{€, }e')$
where all agents pay attention to $\top$, and $Q^{*}=\{(e,e)\colon e\in E\}$.
Then $Y$ is essentially reduced to an epistemic action from $\textsc{noPost}$.
If $E'=E$, $Q'=Q$, $pre'=pre$, $e'=e$, then $Y$ is equivalent
to $X$.
\end{proof}

Jointly, Propositions~\ref{prop:K(M)-M-Equivalence }~and~\ref{prop:AttUp-emulates-NoPost}
show that attention states and actions can emulate the dynamics of
DEL, for epistemic actions in \textsc{noPost}. The converse of Prop.~\ref{prop:AttUp-emulates-NoPost}
is instead false: The dynamics of standard DEL cannot emulate those
of attention states and actions, if one restricts to epistemic actions
in \textsc{noPost}.
\begin{proposition}
\label{prop:no_Equivalent_noPost}There is a \mbox{$X\in \textsc{all}$}
with no equivalent \mbox{$Y\in\textsc{noPost}$}.
\begin{proof}
[Proof sketch] Attention actions may change the valuation of attention
atoms, which epistemic actions in \textsc{noPost} cannot.
\end{proof}

\end{proposition}

If we do not restrict to \textsc{noPost}, then DEL can emulate all
dynamics induced by attention actions:
\begin{proposition}
\label{prop:is_Equivalent_post}For all $X\in\textsc{all}$, there
is an equivalent $Y\in\textsc{Post}$.
\end{proposition}

\begin{proof}
[Proof sketch] Let $X=((E,Q,Q^{*},pre,c),\text{€},e)$. We build
an exponentially larger $Y=((E',Q',pre',post),e)\in\textsc{Post}$.
Let $2^{I}$ be the set of maps $In:I\rightarrow\{0,1\}$. Intuitively,
$In(i)=1$ represents that $i$ has enough attention to pay for $\text{€}_{i}$.
Let $E'=\{f_{In}\colon f\in E,In\in2^{I}\}$. Setting $pre'(f_{In})$
to 
\[
pre(f)\wedge{\textstyle \bigwedge_{In(i)=1}}\alpha_{i}\ge c_{i}(\text{€}_{i},f)\wedge{\textstyle \bigwedge_{In(j)=0}}\alpha_{j}<c_{j}(\text{€}_{j},f)
\]
ensures that $E'$ contains for each $f\in E$ and each $In\in2^{I}$
a unique event $f_{In}$ with $pre'(f_{In})$ satisfied at worlds
that satisfy $pre(f)$ and where exactly agents in $\{i\in I\colon In(i)=1\}$
have attention exceeding the cost of $\text{€}_{i}$. The construction
ensures that for any attention state $(M,w)$, $(w,f_{In})\in K(M,w)^{Y}$
iff ($(w,f)\in(M,w)^{X}$ and\emph{ }$In(i)=1$ iff $A_{i}(w)\geq c_{i}(\text{€}_{i},f)$).
$(M,w)$ and $K(M,w)$ thus have the same number of worlds, and for
all $p\in\Phi$, $(w,f_{In})$ satisfies $p$ iff $(w,f)$ does.

To make $Y$ correctly update the attention atoms, we use that $\mathcal{L}$
can express ``$n$ will be the attention value of agent $i$ after
the event $f$'', writing $next_{f}(\alpha_{i}=n)$ as abbreviation
for
\begin{align*}
((\alpha_{i}=n)\wedge & (\alpha_{i}=\max\{0,n-c_{i}(\text{€}_{i},f)\}))\vee\\
(\neg(\alpha_{i}=n)\wedge & (\alpha_{i}=\min\{n+c_{i}(\text{€}_{i},f),N\})
\end{align*}
The first disjunct reads ``$n$ is $i$'s current attention value
and $c_{i}(\text{€}_{i},f)$\textemdash the cost of what $i$ pays
attention to in $f$\textemdash is zero'' and the second ``$n$
is not the current attention value, but $n$ is the current value
minus $c_{i}(\text{€}_{i},f)$''. We use $next_{f}(\alpha_{i}=n)$
in assigning postconditions, with two key clauses being $post(f_{In})(\alpha_{i}=n)=next_{f}(\alpha_{i}=n)$
and $post(f_{In})(\alpha_{i}<n)={\textstyle \bigvee_{0\le j<n}}next_{f}(\alpha_{i}=j)$,
for all $f_{In}\in E'$.

Finally, we let $(f_{In},g_{In'})\in Q'$ iff either (1) $(f,g)\in Q_{i}$
or (2) $(f,g)\in Q_{i}^{*}$ and $pre'(f_{In})\vDash\alpha_{j}<c_{i}(\text{€}_{i},f)$,
or (3) $(f,g)\in Q_{i}^{*}$ and $pre'(f_{In})\vDash\text{€}_{i}$
iff $pre'(g_{In'})\vDash\text{€}_{i}$. This ensures that $(w,f)$
and $(v,g)$ are related for $i$ in $(M,w)^{X}$ iff $(w,f_{In})$
and $(v,g_{In})$ are related for $i$ in $K(M,w)^{Y}$.

The construction ensures that $(M,w)^{X}$ and $K(M,w)^{Y}$ are isomorphic,
entailing that $X$ and $Y$ are equivalent.
\end{proof}

To finalize our emulation results, we show Prop. \ref{prop:is_Equivalent_post}'s
converse false.

\begin{proposition}
\label{prop:Post_richer_than_All}There exists a $X\in\textsc{Post}$
with no equivalent $Y\in\textsc{all}$.
\end{proposition}

\begin{proof}
[Proof sketch] Epistemic actions may change the valuation of atoms
other than the attention atoms. This cannot be done by attention actions.
\end{proof}

\section{\label{sec:Planning}Epistemic Planning with Attention}

Finally, we turn to epistemic planning with attention. The following
definitions follow those for DEL-based epistemic planning, as in \cite{Bolander2020}.
\begin{definition}
An \emph{attention planning task} $T=(s_{0},\Sigma,\varphi_{g})$
consists of an (initial) attention state $s_{0}$; a finite set of
attention actions $\Sigma$; and a goal formula $\varphi_{g}\in\mathcal{L}$.
A \emph{solution} to $T$ is a finite sequence $X_{1},...,X_{n}$
of actions from $\Sigma$ applicable to $s_{0}$ such that $s_{0}{}^{X_{1},...,X_{n}}\vDash\varphi_{g}$.
\end{definition}

\begin{definition}
For $X\subseteq\textsc{all}$, denote by $\mathcal{T}_{X}$ the class
of attention planning tasks $T=(s_{0},\Sigma,\varphi_{g})$ with $\Sigma\subseteq X$.
Given $\mathcal{T}_{X}\subseteq\mathcal{T}_{\textsc{all}}$, denote
by $\mathsf{PlanEx}\text{-}\mathcal{T}_{X}$ the decision problem
(called the \emph{plan existence problem} on $\mathcal{T}_{X}$):
Given a planning task $T\in\mathcal{T}_{X}$, does $T$ have a solution?
\end{definition}

For DEL-based epistemic planning, the general plan existence problem
is undecidable \cite{BolanderBirkegaard2011}. The same holds when
attention is involved:
\begin{theorem}
\label{thm:All-is-undecidable}$\mathsf{PlanEx}\text{-}\mathcal{T}_{\textsc{all}}$
is undecidable.
\end{theorem}

\begin{proof}
The plan existence problem for DEL epistemic actions with preconditions
of modal depth at most $n$ and without postconditions $\mathsf{PlanEx}\text{-}\mathcal{T}(n,-1)$
is undecidable for $n\geq2$ \cite{Bolander2020}. By Propositions
\ref{prop:static-attention-emulates-Kripke} and \ref{prop:AttUp-emulates-NoPost},
for each epistemic planning task in $\mathsf{PlanEx}\text{-}\mathcal{T}(n,-1)$,
there is an equivalent attention planning task in $\mathsf{PlanEx}\text{-}\mathcal{T}_{\textsc{all}}$,
which is hence undecidable.
\end{proof}

For the class of \textsc{nfl} attention actions, we obtain a more
encouraging result:
\begin{theorem}
\label{thm:NFL-planning-is-decidable}$\mathsf{PlanEx}\text{-}\mathcal{T}_{\textsc{nfl}}$
is decidable.
\end{theorem}

\begin{proof}
Let $T=(s_{0},\Sigma,\varphi_{g})$ be an attention planning task
with $\Sigma\subseteq\text{\textsc{nfl}}$. We show that checking
if $T$ has a solution is decidable by showing the claim that for
any sequence $(X_{k})_{k\in\mathbb{N}}$ of $\Sigma$-actions applicable
to $s_{0}$, there is an $n\in\mathbb{N}$ such that for all $Y\in\Sigma$,
$s_{0}{}^{X_{1},...,X_{n}}\leftrightarroweq s_{0}{}^{X_{1},...,X_{n},Y}$.
Hence, only a finite set of plans needs to be checked to determine
if a solution exists.

We tacitly identify attention states with their bisimulation contractions
(Def.~\ref{def:bisim.contraction}), justified by Lemma~\ref{lem:bisim.contraction.equivalence}.
This makes the reference to the cardinality of sets of worlds meaningful. 

Let $(M,w)=(W,R,V,A,w)$ be any attention state. Applying any $X\in\Sigma$
to $(M,w)$ will either consume attention for some agent (so $A^{X}\neq A$),
or it will not consume any attention ($A=A^{X}$). The claim then
follows from two points:

First, as attention is finite, any applicable sequence of actions
will eventually stop consuming it: For any $X\in\Sigma$, if $A^{X}\neq A$,
then for some $i\in I$, some $w\in W$ and some event $e$ from $X$,
$A_{i}^{X}(v,e)<A_{i}(v)$. Hence, as $\{A_{i}(v)\colon i\in I,v\in W\}$
is bounded, for any sequence $(X_{k})_{k\in\mathbb{N}}$ of $\Sigma$-actions
applicable to $(M,w)$, there is an $n\in\mathbb{N}$ such that for
all $X\in\Sigma$, $A^{X_{1},...,X_{n}}=A^{X_{1},...,X_{n},X}$.

Second, as $W$ is finite, any sequence of actions none of which consume
attention will eventually reach a fixed point: Let $(X_{k})_{k\in\mathbb{N}}$
be a sequence of $\Sigma$-actions applicable to $(M,w)$ such that
$A^{X_{1},...,X_{k}}=A^{X_{1},...,X_{k},X_{k+1}}$. By Lemma \ref{lem:eq.to.pub.ann},
$|W^{X_{1},...,X_{k},X_{k+1}}|<|W^{X_{1},...,X_{k}}|$ (as we identify
models with their bisimulation contraction and as each $X_{k}!$ has
a single event $e^{!}$) or $(M,w)^{X_{1},...,X_{k},X_{k+1}}\leftrightarroweq(M,w)^{X_{1},...,X_{k}}$
(if $M,v\vDash pre(e^{!})$ for all $v\in W$). As $W$ is finite,
the first disjunct can eventually not obtain, so a fixed point is
reached.

The two points jointly imply the claim: for any sequence $(X_{k})_{k\in\mathbb{N}}$
of $\Sigma$-actions applicable to $s_{0}$, the first ensures that
after some $n\geq0$ steps, no more attention is consumed, and the
second implies that after (additional) $m\geq0$ steps, a fixed point
is reached. Hence, finally, if $\varphi_{g}$ is not reached by the
fixed point, $(X_{j})_{j\leq n+m}$ is not a solution.
\end{proof}

\section{\label{sec:Final-Remarks}Final Remarks}

Our main results show that the plan existence problem for epistemic
planning with attention is in general undecidable (Theorem~\ref{thm:All-is-undecidable}),
but decidable for No Free Lunch (\textsc{nfl}) actions (Theorem~\ref{thm:NFL-planning-is-decidable}).
As \textsc{nfl} actions are of special interest, this is an encouraging
result for epistemic planning with attention.

As (DEL-based) epistemic planning is in general undecidable, the emulation
results of Sec.~\ref{sec:Emulation-Results} makes Theorem~\ref{thm:All-is-undecidable}
unsurprising. However, Theorem~\ref{thm:NFL-planning-is-decidable}
strikes a sharp contrast with other results in epistemic planning:

Epistemic planning task classes have earlier been investigated by
number of agents, with the single-agent case decidable \cite{BolanderBirkegaard2011},
and the $n$-agent $n\geq2$ cases undecidable \cite{Aucher2013}
(both without postconditions). Ensuing, decidable fragments have been
sought in the hierarchy of classes $\mathcal{T}(m,n)$ allowing epistemic
actions with modal depth of pre- and postconditions at most $m$ and
$n$, respectively. \cite{Bolander2020} provide a recent survey,
contributing so that the only open question concerns $\mathcal{T}(1,-1)$,
`$-1$' referring to no postconditions. Of special interest to this
paper are the classes that attention planning can emulate ($\mathcal{T}(m,0)$)
or can be emulated by ($\mathcal{T}(m,1)$). The case for $\mathcal{T}(0,0)$
is decidable, but undecidable for $\mathcal{T}(m,0),m\geq1$ and $\mathcal{T}(m,1),m\geq0$
\cite{Bolander2020}. However, the emulations of Sec.~\ref{sec:Emulation-Results}
makes it evident that the class $\mathcal{T}_{\textsc{nfl}}^{\textsc{del}}$
of epistemic planning tasks obtained by emulating $\mathcal{T}_{\textsc{nfl}}$
is a proper subclass of $\bigcup_{m\in\mathbb{N}}\mathcal{T}(m,0)\cup\mathcal{T}(m,1)$
with $\mathcal{T}_{\textsc{nfl}}^{\textsc{del}}\cap(\mathcal{T}(m,0)\cup\mathcal{T}(m,1))$
non-empty for every $m\in\mathbb{N}$. By Theorem~\ref{thm:NFL-planning-is-decidable},
the plan existence problem for $\mathcal{T}_{\textsc{nfl}}^{\textsc{del}}$
is decidable, for any number of agents. Hence, the paper's results
finds a class of epistemic planning tasks that cuts across the oft
studied classes, for which the plan existence problem is decidable,
and which, by allowing arbitrary preconditions, is still reasonably
expressive.\medskip{}

\noindent We conclude with remarks on future research. First, regarding
complexity, then we venture that the plan existence problem for $\mathcal{T}_{\textsc{nfl}}$
is NP-complete, as this is the case for epistemic planning with public
announcements \cite{Bolander2015}. Second, we believe $\mathcal{L}$
is sufficiently expressive to obtain a complete axiom system for attention
states. This, the emulations and existing completeness results for
epistemic actions with postconditions \cite{Ditmarsch_Kooi_ontic}
would yield a complete dynamic logic for attention actions. However,
the emulations enforce an exponential blowup that would spill into
reduction axioms. It may therefore be desirable to define such directly
for attention actions. Finally: The model assumes common knowledge
of what agents pay attention to. This may be dropped by adding indistiguishability
between attention maps, at the cost of a more elaborate product. We
have left this construction for longer work.

\subsubsection*{Acknowledgments.}

We thank the Center for Information and Bubble Studies, funded by
the Carlsberg Foundation. RKR was partially supported by the DFG-ANR
joint project \emph{Collective Attitude Formation} {[}RO 4548/8-1{]}.

\vfill{}

\pagebreak{}

\begin{center}
	\textbf{\LARGE{}Epistemic Planning with Attention as a Bounded Resource}\\
	\textbf{\LARGE{}Proof Appendix}{\LARGE\par}
	\par\end{center}

\medskip{}

\noindent This file contains the proposition statements and full proofs
for all proof sketches.

\medskip{}

Throughout the file, we write $M^{X}=(W^{X},R^{X},V^{X},A^{X}),$where,
by Def.~\ref{def: Attention Update}, $(M,w)^{X}=((W^{X},R^{X},V^{X},A^{X}),w^{X})$.
Similarly, where $K(M,w)^{Y}=((W^{Y},R^{Y},V_{A}^{Y}),w^{Y})$, we
write $K(M)^{Y}=(W^{Y},R^{Y},V_{A}^{Y})$.

\medskip{}

\noindent \textbf{Lemma~\ref{lem:bisim.contraction.equivalence}.}
For any attention state $(M,w)$, for all $\varphi\in\mathcal{L}$,
$(M,w)_{\leftrightarroweq}\vDash\varphi$ iff $(M,w)\vDash\varphi$.
\begin{proof}
	Clearly, $(M,w)_{\leftrightarroweq}=((W',R',V',A'),[w])$ is an attention
	state. Moreover, $(M,w)_{\leftrightarroweq}=((W',R',V',A'),[w])$
	is bisimilar to $(M,w)=((W,R,V,A),w)$, witnessed by $Z=\{(w,[w])\colon w\in W\}$.
	We show that this is the case by showing that $Z$ is a bisimulation:
	Let \textsc{$vZ[v]$.}
	\begin{description}
		\item [{(Atoms)}] By Def.~\ref{def:bisim.contraction}, $v\in V(p)\in(M,w)$
		iff $[v]\in V'(p)\in(M,w)_{\leftrightarroweq}$ for all $p\in\Phi$,
		and $A_{i}(v)=A'_{i}([v])$ for all $i\in I$;
		\item [{(Forth)}] Assume $vR_{i}u$ for some $i\in I$. By Def.~\ref{def:bisim.contraction},
		$[v]R'_{i}[u]$, and by construction of $Z$, $uZ[u]$.
		\item [{(Back)}] Assume $[v]R'_{i}[u]$ for some $i\in I$. By Def.~\ref{def:bisim.contraction},
		there exists a $v\in[v]$, and a $u\in[u]$, with $vR_{i}u$. By construction
		of $Z$, $uZ[u]$.
	\end{description}
	Hence, $(M,w)\leftrightarroweq(M,w)_{\leftrightarroweq}$, and the
	conclusion follows by Prop.~\ref{prop:Hennessy-Milner}.
\end{proof}

\noindent \textbf{Proposition \ref{prop:K(M)-M-Equivalence }.} For
any attention state $(M,w)$, $K(M,w)$ is
\begin{description}
	\item [{$i)$}] \noindent an epistemic state; 
	\item [{$ii)$}] \noindent isomorphic to $(M,w)$; 
	\item [{$iii)$}] \noindent for all $\varphi\in\mathcal{L}$, $K(M,w)\vDash\varphi$
	iff $(M,w)\vDash\varphi$.
\end{description}
\begin{proof}
	Let $(M=(W,R,V),w)$ be an attention state for $At=\Phi\cup\Psi$,
	and $K(M,w)=(W,R,V_{A})$ be the Kripke rendition of $(M,w)$.
	\begin{description}
		\item [{$i)$}] $K(M,w)$ is an epistemic Kripke model for $At$: $W$
		is non-empty; $R$ assigns equivalence relations on $W$, and $V_{A}$
		is clearly a valuation for $At$. 
		\item [{$ii)$}] Let $f:W\rightarrow W$ be the identity function. Clearly,
		$f$ is a bijection, and clearly $(w,v)\in R_{i}$ iff $(f(w),f(v))\in R_{i}$
		for any $i\in I$. Lastly, for $p\in At$, $M,w\vDash p$ iff $w\in V_{A}(p)$
		iff $K(M,w)\vDash p$. Hence, $K(M,w)$ is isomorphic to $(M,w)$.
		\item [{$iii)$}] Follows from $ii)$.
	\end{description}
\end{proof}

\noindent \textbf{Proposition \ref{prop:AttUp-emulates-NoPost}.}
For all $X\in\textsc{noPost}$, there is an equivalent $Y\in\textsc{all}$.
\begin{proof}
	Let $X=(\mathcal{E},e)=((E,Q,pre),e)\in\textsc{noPost}$ and define
	$Y=(\mathcal{A},\text{€}, e')=((E',Q',Q^{*},pre',c),\text{€}, e')\in\textsc{all}$
	such that $E'=E$, $Q'=Q$, $Q^{*}=\{(f,f)\colon f\in E\}$, $pre'=pre$,
	$e'=e$, $c_{i}(\varphi,f)=0$, for all $\varphi\in\mathcal{L}$,
	$f\in E'$, $i\in I$, and $\text{€}_{i}=\top$ for all $i\in I$.
	Let $(M,w)=((W,R,V,A),w)$ be an attention state and $K(M,w)=((W,R,V_{A}),w)$
	be its Kripke rendition.
	
	The proof is by induction on formula complexity.\emph{ }
	
	\noindent \emph{Base: }The base contains three cases $i)$ $\varphi:=p\in\Phi$,
	$ii)$ $\varphi:=q\in\Psi$, and $iii)$ $\varphi:=\top$. 
	\begin{description}
		\item [{$i)\,\,\varphi:=p\in\Phi$}] $K(M)^{X},(w,e)\vDash p$ iff (Def.~\ref{def: standard product updates})
		$(w,e)\in V_{A}^{X}(p)$ iff (Def.~\ref{def: standard product updates})
		$w\in V_{A}(p)$ iff (Def.~\ref{def: semantics for epistemic states})
		$K(M,w)\vDash p$ iff (Prop.~\ref{prop:K(M)-M-Equivalence }) $M,w\vDash p$
		iff (Def.~\ref{def: semantics attention language}) $w\in V(p)$
		iff (Def.~\ref{def: Attention Update}) $(w,e)\in V^{Y}(p)$ iff
		(Def.~\ref{def: semantics attention language}) $M^{Y},(w,e)\vDash p$. 
		\item [{$ii)\,\,\varphi:=q\in\Psi$}] Either $q:=(\alpha_{i}<n)$ or $q:=(\alpha_{i}=n)$.
		$K(M)^{X},(w,e)\vDash\alpha_{i}<n$ iff (Def.~\ref{def: semantics for epistemic states})
		$(w,e)\in V^{X}(\alpha_{i}<n)$ iff (Def.~\ref{def: standard product updates})
		$w\in V_{A}(\alpha_{i}<n)$ iff (Def.~\ref{def:K-transform}) $M,w\vDash\alpha_{i}<n$
		iff (Def.~\ref{def: semantics attention language}) $A_{i}(w)<n$
		iff (Def.~\ref{def: Attention Update} and construction of $c$)
		$A_{i}^{Y}(w,e)=\max\{0,A{}_{i}(w)-c_{i}(\text{€}_{i},e)\}=\max\{0,A{}_{i}(w)-0\}=A_{i}(w)$
		iff $A_{i}^{Y}(w,e)<n$ iff (Def.~\ref{def: semantics attention language})
		$M^{Y},(w,e)\vDash\alpha_{i}<n$. Analogously for $q:=(\alpha_{i}=n)$.
	\end{description}
	\begin{lyxlist}{00.00.0000}
		\item [{$iii)\,\,\varphi:=\top$}] is trivial.
	\end{lyxlist}
	\emph{Step}. Assume $\psi,\chi\in\mathcal{L}$ satisfy the statement.
	The cases for $\varphi:=\neg\psi$ and $\varphi:=\psi\wedge\chi$
	are straightforward. The remaining case is $\varphi:=K_{i}\psi$.
	We first show that $i)\,\,W^{X}=W^{Y}$ and $ii)\,\,R^{X}=R^{Y}$.
	\begin{description}
		\item [{$i)$}] By Prop. \ref{prop:K(M)-M-Equivalence }, $K(M,w)$ is
		isomorphic to $(M,w)$, and by construction, $pre=pre'$. Then for
		$w\in W$, $f\in E=E'$, $M,w\vDash pre(f)$ iff $K(M,w)\vDash pre'(f)$,
		i.e., $(w,f)\in W^{X}$ iff $(w,f)\in W^{Y}$, by Def.~\ref{def: standard product updates}
		and Def.~\ref{def: Attention Update}.
		\item [{$ii)$}] We show $R^{X}=R^{Y}$ by showing $R^{X}\subseteq R^{Y}$
		and $R^{Y}\subseteq R^{X}$.
		\begin{description}
			\item [{$(R^{X}\subseteq R^{Y})$}] If $((w,e),(v,f))\in R_{i}^{Y}$, for
			some $i\in I$, then, by Def.~\ref{def: Attention Update}, $(w,v)\in R_{i}$,
			and either $(e,f)\in(Q_{i}'\cup Q_{i}^{*})$, or $(e,f)\in(Q_{i}'\cup Q_{i}^{*}\setminus[\text{€}_{i}\mid\neg\text{€}_{i}])$,
			i.e., either $(1)$ $(e,f)\in Q_{i}'$, or $(2)$ $(e,f)\in Q_{i}^{*}$,
			or $(3)$ $(e,f)\in Q_{i}^{*}\setminus[\text{€}_{i}\mid\neg\text{€}_{i}]$. 
			\begin{description}
				\item [{$(1)$}] If $(e,f)\in Q_{i}'$, then $(e,f)\in Q_{i}$, as $Q'_{i}=Q_{i}$,
				by def. of $Q'_{i}$. 
				\item [{$(2)$}] If $(e,f)\in Q_{i}^{*}$: Since $Q'_{i}=Q_{i}$, and $Q_{i}$
				is an equivalence relation, then $Q'_{i}$ is an equivalence relation.
				Then, by construction of $Q_{i}^{*}$, if $(e,f)\in Q_{i}^{*}$, then
				$(e,f)\in Q'_{i}$, i.e., $Q^{*}\subseteq Q'_{i}$. As $Q'_{i}=Q_{i}$,
				then $(e,f)\in Q_{i}$. 
				\item [{$(3)$}] If $(e,f)\in Q_{i}^{*}\setminus[\text{€}_{i}\mid\neg\text{€}_{i}]\subseteq Q_{i}^{*}$,
				then by analogous reasoning to $(2)$, $(e,f)\in Q_{i}$. 
				
				As by Def.~\ref{def:K-transform}, $(w,v)\in R_{i}\in(M,w)$ iff
				$(w,v)\in R_{i}\in K(M,w)$, in all three cases $((w,e),(v,f))\in R_{i}^{X}$,
				by Def.~\ref{def: standard product updates}.
			\end{description}
			\item [{$(R^{Y}\subseteq R^{X})$}] If $((w,e),(v,f))\in R_{i}^{X}$, for
			some $i\in I$, then by Def.~\ref{def: standard product updates},
			$(w,v)\in R_{i}$ and $(e,f)\in Q_{i}$. By $Q'_{i}=Q_{i}$, $(e,f)\in Q'_{i}$,
			and by Def.~\ref{def:K-transform} $(w,v)\in R_{i}\in(M,w)$ iff
			$(w,v)\in R_{i}\in K(M,w)$, so $((w,e),(v,f))\in R_{i}^{Y}$, by
			Def.~\ref{def: Attention Update}.
		\end{description}
	\end{description}
	Now consider $\varphi:=K_{i}\psi$. $K(M)^{X},(w,e)\vDash K_{i}\psi$
	iff (Def.~\ref{def: semantics for epistemic states}) for all $(v,f)\in W^{X}$
	such that $((w,e),(v,f))\in R_{i}^{X}$, $K(M)^{X},(v,f)\vDash\psi$
	iff (by $W^{X}=W^{Y}$, $R^{X}=R^{Y}$, and inductive hypothesis)
	for all $(v,f)\in W^{Y}$ such that $((w,e),(v,f))\in R_{i}^{Y}$,
	$M^{Y},(v,f)\vDash\psi$ iff (Def.~\ref{def: semantics attention language})
	$M^{Y},(w,e)\vDash K_{i}\psi$. 
\end{proof}

Hence, for all $\varphi\in\mathcal{L}$, $K(M)^{X},(w,e)\vDash\varphi$
iff $M^{Y},(w,e)\vDash\varphi$.

\noindent \textbf{Proposition \ref{prop:no_Equivalent_noPost}.} There
is a $X\in\textsc{all}$ for which there is no equivalent $Y\in\textsc{noPost}$.
\begin{proof}
	Let $I=\{i\}$, and $(M,w)=((W,R,V,A),w)$ have $W=\{w,v\}$, $(w,v)\in R_{i}$,
	$w\in V(p)$, $A_{i}(w)=A_{i}(v)=k$ for some $1\leq k<N$. Then $M,w\vDash(\alpha_{i}=k)\wedge\neg(\alpha_{i}=k-1)$.
	
	Let $X=(\mathcal{A},\text{€},e)=((E,Q,Q^{*},pre,c),\text{€},e)\in\textsc{all}$
	have $E=\{e,f\}$, $(e,f)\in Q=Q_{i}^{*}$, $pre(e)=pre(f)=p$, $c_{i}(p,e)=c_{i}(p,f)=1$,
	and $\text{€}_{i}=p$. Then, by Def.~\ref{def: Attention Update},
	$M^{X},(w,e)\vDash(\alpha_{i}=k-1)$.
	
	However, there does not exist a $Y=(\mathcal{E},e')=((E',Q',pre',post),e')\in\textsc{noPost}$
	for which $K(M)^{Y},(w,e')\vDash(\alpha_{i}=k-1)$: As $post(e')(\alpha_{i}=k-1)=(\alpha_{i}=k-1)$,
	for all $e'\in E'$, then by Def.~\ref{def: standard product updates},
	$K(M)^{Y},(w,e')\vDash(\alpha_{i}=k-1)$ iff $K(M,w)\vDash(\alpha_{i}=k-1)$.
	As $M,w\vDash\neg(\alpha_{i}=k-1)$, by Prop. \ref{prop:K(M)-M-Equivalence },
	$K(M,w)\vDash\neg(\alpha_{i}=k-1)$. Hence $K(M)^{Y},(w,e)\vDash\neg(\alpha_{i}=k-1)$,
	so $Y$ is not equivalent to $X$.
\end{proof}

\noindent \textbf{Proposition \ref{prop:is_Equivalent_post}.} For
all $X\in\textsc{all}$, there is an equivalent $Y\in\textsc{Post}$.
\begin{proof}
	Let an attention action $X=(\mathcal{A},\text{€}, e)=((E,Q,Q^{*},pre,c),\text{€}, e)\in\textsc{all}$
	be given. We define an epistemic action $Y\in\textsc{Post}$ and show
	$Y$ equivalent to $X$. Let $Y=(\mathcal{E},e)=((E',Q',pre',post),e)$
	with 
	\begin{description}
		\item [{$E'$}] $=\{f_{In}:f\in E,In\in2^{I}\}$;
		\item [{$Q'$}] is such that, for all $i\in I$, for all $In,In'\in2^{I}$,
		$(f_{In},g_{In'})\in Q'_{i}$ iff 
		\begin{description}
			\item [{$i)$}] $(f,g)\in Q_{i}$, or
			\item [{$ii)$}] $(f,g)\in Q_{i}^{*}$ and $In(i)=In'(i)=0$, or 
			\item [{$iii)$}] $(f,g)\in Q_{i}^{*}$ and ($pre(f)\vDash\text{€}_{i}$
			iff $pre(g)\vDash\text{€}_{i}$)
		\end{description}
		\item [{$pre'$}] is such that for all $f\in E$ and all $f_{In}\in E'$,
		\begin{align*}
			pre'(f_{In}) & =pre(f)\wedge{\textstyle \bigwedge_{In(i)=1}}(\alpha_{i}\ge c_{i}(\text{€}_{i},f))\wedge{\textstyle \bigwedge_{In(j)=0}}(\alpha_{j}<c_{j}(\text{€}_{j},f))
		\end{align*}
		\item [{$post$}] is such that for all $f_{In}\in E'$, $i\in I$, for
		all $p\in\Phi$, and all $n\in\{1,...,N\}$
		\begin{align*}
			post(f_{In})(p) & =p\\
			post(f_{In})(\alpha_{i}<0) & =\bot\\
			post(f_{In})(\alpha_{i}=0) & ={\textstyle \bigvee_{n\le c_{i}(\text{€}_{i},f)}}(\alpha_{i}=n)\\
			post(f_{In})(\alpha_{i}=n) & =next_{f}(\alpha_{i}=n)\\
			post(f_{In})(\alpha_{i}<n) & ={\textstyle \bigvee_{0\le j<n}}next_{f}(\alpha_{i}=j)
		\end{align*}
		where (see below for intuition):
		\begin{flalign*}
			next_{f}(\alpha_{i}=n):= & ((\alpha_{i}=n)\wedge(\alpha_{i}=\max\{0,n-c_{i}(\text{€}_{i},f)\}))\vee\\
			& (\neg(\alpha_{i}=n)\wedge(\alpha_{i}=\min\{n+c_{i}(\text{€}_{i},f),N\})
		\end{flalign*}
	\end{description}
	We emulate the change of attention atoms by the postconditions involving
	$next_{f}(\alpha_{i}=n)$, intuitively read ``$n$ will be the attention
	value of agent $i$ after the event $f$''. The first disjunct reads
	``$n$ is $i$'s current attention value and $c_{i}(\text{€}_{i},f)$\textemdash the
	cost of what $i$ pays attention to in $f$\textemdash is zero''
	and the second ``$n$ is not the current attention value, but $n$
	is the current value minus $c_{i}(\text{€}_{i},f)$''.
	
	We show $X$ and $Y$ equivalent by showing that for any attention
	state $(M,w)$ and its Kripke rendition $K(M,w)$, $(M,w)^{X}\vDash\varphi$
	iff $K(M,w)^{Y}\vDash\varphi$, for all $\varphi\in\mathcal{L}$. 
	For any attention state $(M,w)=((W,R,V,A),w)$ and its Kripke rendition
	$K(M,w)=((W,R,V_{A}),w)$, $(M,w)^{X}$ is isomorphic to $K(M,w)^{Y}$.
	
	Recall that we let $M^{X}$ denote $(W^{X},R^{X},V^{X},A^{X})$ and
	$K(M)^{Y}=(W^{Y},R^{Y},V_{A}^{Y})$.
	To establish the claim, it must be shown that (1) there exists a bijection
	$\mathfrak{g}:W^{X}\longrightarrow W^{Y}$; (2) for all $x,y\in W^{X}$,
	all $i\in I$, $(x,y)\in R_{i}^{X}$ iff $(\mathfrak{g}(x),\mathfrak{g}(y))\in R_{i}^{Y}$;
	(3) for all $(v,f)\in W^{X}$, $M^{X},(v,f)\vDash p$ iff $K(M)^{Y},\mathfrak{g}(v,f)\vDash p$,
	for all $p\in At=\Phi\cup\Psi$.
	\begin{description}
		\item [{(1)}] Let $\mathfrak{g}:S(W^{X})\longrightarrow S(W^{Y})$ with
		$\mathfrak{g}(v,f)=(v,f_{In})$. It is shown that $\mathfrak{g}$
		is a bijection by showing that $i)$ $\mathfrak{g}$ is well-defined;
		$ii)$ $\mathfrak{g}$ is injective; $iii)$ $\mathfrak{g}$ is surjective. 
		\begin{description}
			\item [{$i)$}] $\mathfrak{g}$ is well-defined: $(v,f)\in W^{X}$ iff
			(Def.~\ref{def: Attention Update}) $M,v\vDash pre(f)$ iff (Prop.~\ref{prop:K(M)-M-Equivalence })
			$K(M,v)\vDash pre(f)$. Moreover, for each $i\in I$, either $M,v\vDash(\alpha_{i}\ge c_{i}(\text{€}_{i},f))$
			or $M,v\vDash\alpha_{i}<c_{i}(\text{€}_{i},f)$, but not both.
			Hence, there exists a unique $In\in2^{I}$ for which $M,w\vDash(\alpha_{i}\ge c_{i}(\text{€}_{i},f))$
			iff $In(i)=1$. By Prop.~\ref{prop:K(M)-M-Equivalence }, the same
			holds for $K(M,v)$. By construction of $E'$ and as $K(M,v)\vDash pre(f)$,
			there thus exists a unique $In\in2^{I}$ for which $K(M,w)\vDash pre'(f_{In})$.
			Hence, $(v,f_{In})\in W^{Y}$ exists and is unique.
			\item [{$ii)$}] $\mathfrak{g}$ is injective: Let $(v,f),(u,g)\in W^{X}$,
			such that $(v,f)\not=(u,g)$. Then $\mathfrak{g}(v,f)=(v,f_{In})$
			and $\mathfrak{g}(u,g)=(u,g_{In'})$ for some $In,In'\in2^{I}$. As
			$(v,f)\not=(u,g)$, either $v\neq u$ or $f\neq g$. Hence $(v,f_{In})\neq(u,g_{In'})$
			as either $v\neq u$ or $f_{In}\neq g_{In'}$.
			\item [{$iii)$}] $\mathfrak{g}$ is surjective: Let $(v,f_{In})\in W^{Y}$.
			Then $K(M,v)\vDash pre'(f_{In})$. By Prop.~\ref{prop:K(M)-M-Equivalence },
			also $M,v\vDash pre(f_{In})$, and hence $M,v\vDash pre(f)$. Hence
			$(v,f)\in W^{X}$. By definition of $\mathfrak{g}$, $\mathfrak{g}(v,f)=(v,f_{In}).$
		\end{description}
		Thus the map $\mathfrak{g}$ is bijective.
		\item [{(2)}] $((w,f),(v,g))\in R_{i}^{X}$ iff (Def.~\ref{def: Attention Update})
		$(w,v)\in R_{i}$ and either $i)$ $(f,g)\in Q_{i}$, or $ii)$ $(f,g)\in Q_{i}^{*}$
		and $A_{i}(w)<c_{i}(\text{€}_{i},f)$, or $iii)$ $(f,g)\in Q_{i}^{*}$
		and ($pre(f)\vDash\text{€}_{i}$ iff $pre(g)\vDash\text{€}_{i}$).
		Then,
		\begin{description}
			\item [{$i)$}] if $(f,g)\in Q_{i}$: By construction of $E'$, $(f_{In},g_{In'})\in Q'_{i}$
			for all $In,In'\in2^{I}$, implying, by Def.~\ref{def: standard product updates},
			$((w,f_{In}),(v,g_{In'}))\in R_{i}^{Y}$, i.e., $(\mathfrak{g}(w,f),\mathfrak{g}(v,g))\in R_{i}^{Y}$;
			\item [{$ii)$}] if $(f,g)\in Q_{i}^{*}$ and $A_{i}(w)<c_{i}(\text{€}_{i},f)$:
			Recall that as $(w,v)\in R_{i}$, $A_{i}(w)=A_{i}(v)$ and as $(f,g)\in Q_{i}^{*}$,
			$c_{i}(\text{€}_{i},f)=c_{i}(\text{€}_{i},g)$. As \textsc{$(w,f)\in W^{X}$},
			also $(w,f_{In})\in W^{Y}$ for some $In$ with $In(i)=0$. Similarly,
			$(v,g_{In'})\in W^{Y}$ for some $In'$ with $In'(i)=0$. Hence $(f_{In},g_{In'})\in Q_{i}'$.
			By Def.~\ref{def: standard product updates}, $((w,f_{In}),(v,g_{In'}))\in R_{i}^{Y}$,
			i.e., $(\mathfrak{g}(w,f),\mathfrak{g}(v,g))\in R_{i}^{Y}$;
			\item [{$iii)$}] $(f,g)\in Q_{i}^{*}$ and ($pre(f)\vDash\text{€}_{i}$
			iff $pre(g)\vDash\text{€}_{i}$): By (1) above, as $(w,f),(v,g)\in W^{X}$,
			also $(w,f_{In}),(v,g_{In'})\in W^{Y}$ for some $In,In'$$\in2^{I}$.
			Then directly, $(f_{In},g_{In'})\in Q_{i}'$. As $(w,v)\in R_{i}$,
			also $((w,f_{In}),(v,g_{In'}))\in R_{i}^{Y}$, i.e., $(\mathfrak{g}(w,f),\mathfrak{g}(v,g))\in R_{i}^{Y}$.
		\end{description}
		Hence $((w,f),(v,g))\in R_{i}^{X}$ iff $(\mathfrak{g}(w,f),\mathfrak{g}(v,g))\in R_{i}^{Y}$.
		\item [{(3)}] Let $(v,f)\in W^{X}$. There are two main cases: $p\in\Phi$
		and $q\in\Psi$.
		\begin{description}
			\item [{$p\in\Phi$.}] $M^{X},(v,f)\vDash p$ iff (Def.~\ref{def: Attention Update})
			$M,v\vDash p$ iff (Prop.~\ref{prop:K(M)-M-Equivalence }) $K(M,v)\vDash p$
			iff (def. of $post\in E'$ and Def.~\ref{def: standard product updates})
			$K(M)^{Y},\mathfrak{g}(v,f)\vDash p$.
			\item [{$q\in\Psi$.}] There are four subcases: $i)$ $q:=(\alpha_{i}=0)$;
			$ii)$ $q:=(\alpha_{i}=n)$, for $n>0$; $iii)$ $q:=(\alpha_{i}<0)$;
			$iv)$ $q:=(\alpha_{i}<n)$, for $n>0$. 
			\begin{description}
				\item [{$i)$}] $M^{X},(v,f)\vDash(\alpha_{i}=0)$ iff (Def.~\ref{def: semantics attention language})
				$A_{i}^{X}(v,f)=0$ iff (Def.~\ref{def: Attention Update}) $A_{i}^{X}(v,f)=\max\{0,A{}_{i}(v)-c_{i}(\text{€}_{i},f)\}=0$
				iff $A_{i}(v)=n$ with $n\le c_{i}(\text{€}_{i},f)$ and $M,v\vDash(\alpha_{i}=n)$
				iff (Prop.~\ref{prop:K(M)-M-Equivalence }) $K(M,v)\vDash(\alpha_{i}=n)$
				with $n\le c_{i}(\text{€}_{i},f)$, i.e., $K(M,v)\vDash\bigvee_{n\le c_{i}(\text{€}_{i},f)}(\alpha_{i}=n)$
				iff (def. of $post$ and Def.~\ref{def: standard product updates})
				$K(M)^{Y},(v,f)\vDash(\alpha_{i}=0)$. 
				\item [{$ii)$}] $M^{X},(v,f)\vDash(\alpha_{i}=n)$ with $n>0$ iff (Def.~\ref{def: semantics attention language})
				$A_{i}(v,f)^{X}=n>0$ iff (Def.~\ref{def: Attention Update}) $A_{i}^{X}(v,f)=\max\{0,A{}_{i}(v)-c_{i}(\text{€}_{i},f)\}=n>0$.
				Two cases: $(1)$ $c_{i}(\text{€}_{i},f)=0$, or $(2)$ $c_{i}(\text{€}_{i},f)>0$. 
				\begin{description}
					\item [{$(1)$}] $c_{i}(\text{€}_{i},f)=0$: Then $A_{i}^{X}(v,f)=\max\{0,A{}_{i}(v)-0\}=A_{i}(v)=n$
					iff (Def.~\ref{def: semantics attention language}) $M,v\vDash(\alpha_{i}=n)$
					and $M,v\vDash(\alpha_{i}=\max\{0,n-c_{i}(\text{€}_{i},f)\})$
					iff (Prop.~\ref{prop:K(M)-M-Equivalence } and Def.~\ref{def: semantics for epistemic states})
					$K(M,v)\vDash(\alpha_{i}=n)\wedge(\alpha_{i}=\max\{0,n-c_{i}(\text{€}_{i},f)\})$
					iff (Def.~\ref{def: semantics for epistemic states} and def. of
					$next_{f}$) $K(M,v)\vDash next_{f}(\alpha_{i}=n)$ iff (def. of $post$
					in $E'$ and Def.~\ref{def: standard product updates}) $K(M)^{Y},(v,f)\vDash(\alpha_{i}=n)$,
					for $n>0$. 
					\item [{$(2)$}] $c_{i}(\text{€}_{i},f)>0$: Then $A_{i}(v)=k$,
					with $n<k\le N$, and $k=\min\{n+c_{i}(\text{€}_{i},f),N\}$
					iff (Def.~\ref{def: semantics attention language}) $M,v\vDash(\alpha_{i}=\min\{n+c_{i}(\text{€}_{i},f),N\})$
					and $M,v\vDash\neg(\alpha_{i}=n)$ iff (Prop.~\ref{prop:K(M)-M-Equivalence }
					and Def.~\ref{def: semantics for epistemic states}) $K(M,v)\vDash(\alpha_{i}=\min\{n+c_{i}(\text{€}_{i},f),N\})\wedge\neg(\alpha_{i}=n)$
					iff (Def.~\ref{def: semantics for epistemic states} and def. of
					$next_{f}$) $K(M,v)\vDash next_{f}(\alpha_{i}=n)$ iff (def. of $post$
					in $E'$ and Def.~\ref{def: standard product updates}) $K(M)^{Y},(v,f)\vDash(\alpha_{i}=n)$,
					for $n>0$. 
				\end{description}
				\item [{$iii)$}] $(M)^{X},(v,f)\vDash(\alpha_{i}<0)$ iff (Def.~\ref{def: semantics attention language})
				$A_{i}(v)^{X}<0$, which can never be the case, so $(M)^{X},(v,f)\vDash\bot$
				iff (Def.~\ref{def: Attention Update}) $M,v\vDash\bot$ iff (Prop.~\ref{prop:K(M)-M-Equivalence })
				$K(M,v)\vDash\bot$ iff (def. of $post$ in $E'$ and Def.~\ref{def: standard product updates})
				$K(M)^{Y},(v,f)\vDash(\alpha_{i}<0)$.
				\item [{$iv)$}] $(M)^{X},(v,f)\vDash(\alpha_{i}<n)$, for $n>0$ iff (Def.~\ref{def: semantics attention language})$A_{i}(v,f)^{X}<n>0$
				iff (Def.~\ref{def: Attention Update}) $A_{i}^{X}(v,f)=\max\{0,A{}_{i}(v)-c_{i}(\text{€}_{i},f)\}=j$,
				with $0\le j<n$. Two cases: $(1)$ $c_{i}(\text{€}_{i},f)=0$;
				or $(2)$ $c_{i}(\text{€}_{i},f)>0$. 
				\begin{description}
					\item [{$(1)$}] $c_{i}(\text{€}_{i},f)=0$: Then $A_{i}^{X}(v,f)=\max\{0,A{}_{i}(v)-0\}=A_{i}(v)=j$
					with $0\le j<n$ iff (Def.~\ref{def: Attention Update} and Def.~\ref{def: semantics attention language})
					$M,v\vDash\bigvee_{0\le j<n}((\alpha_{i}=j)\wedge(\alpha_{i}=\max\{0,j-c_{i}(\text{€}_{i},f)\}))$
					iff (Prop.~\ref{prop:K(M)-M-Equivalence }) $K(M,v)\vDash\bigvee_{0\le j<n}((\alpha_{i}=j)\wedge(\alpha_{i}=\max\{0,j-c_{i}(\text{€}_{i},f)\}))$
					iff (Def.~\ref{def: semantics for epistemic states} and def. of
					$next_{f}$) $K(M,v)\vDash\bigvee_{0\le j<n}next_{f}(\alpha_{i}=j)$
					iff (def. of $post$ in $E'$ and Def.~\ref{def: standard product updates})
					$K(M)^{Y},(v,f)\vDash(\alpha_{i}<n)$, for $n>0$.
					\item [{$(2)$}] $c_{i}(\text{€}_{i},f)>0$: Then $A_{i}(v)=k$, with $k=\min\{j+c_{i}(\text{€}_{i},f),N\}$
					and $0\le j<n$ iff
					
					(Def.~\ref{def: semantics attention language}) $M,v\vDash\bigvee_{0\le j<n}(\neg(\alpha_{i}=j)\wedge(\alpha_{i}=\min\{j+c_{i}(\text{€}_{i},f),N\}))$
					iff (Prop.~\ref{prop:K(M)-M-Equivalence }) $K(M,v)\vDash\bigvee_{0\le j<n}(\neg(\alpha_{i}=j)\wedge(\alpha_{i}=\min\{j+c_{i}(\text{€}_{i},f),N\}))$
					iff (Def.~\ref{def: semantics for epistemic states} and def. of
					$next_{f}$) $K(M,v)\vDash{\textstyle \bigvee_{0\le j<n}}next_{f}(\alpha_{i}=j)$
					iff (def. of $post$ in $E'$ and Def.~\ref{def: standard product updates})
					$K(M)^{Y},(v,f)\vDash(\alpha_{i}<n)$, for $n>0$. 
					
				\end{description}
			\end{description}
		\end{description}
	\end{description}
	Hence, $(M,w)^{X}$ is isomorphic to $K$$(M,w)^{Y}$, which by standard
	arguments implies that for all $\varphi\in\mathcal{L}$, $(M,w)^{X}\vDash\varphi$
	iff $K$$(M,w)^{Y}\vDash\varphi$. Hence, $X$ is equivalent to $Y$.
\end{proof}

\noindent \textbf{Proposition \ref{prop:Post_richer_than_All}.} There
exists a $X\in\textsc{Post}$ for which there does not exist an equivalent
$Y\in\textsc{all}$.
\begin{proof}
	Let $I=\{i\}$ and let $(M,w)=((W,R,V,A),w)$ be an attention state
	with $W=\{w\}$, $(w,w)\in R_{i}$, $w\in V(p)$ for $p\in\Phi$,
	$A_{i}(w)=0$. Let $K(M,w)$ be its Kripke rendition. Then both $(M,w)$
	and $K(M,w)$ satisfy $p$.
	
	Let be given the epistemic action $X=(\mathcal{E},e')=((E',Q',pre',post),e')\in\textsc{Post}$
	have $E'=\{e'\}$, $Q'=\{(e',e')\}$, $pre'(e')=\top$, $post(e',p)=\neg p$,
	for all $p\in\Phi$, and $post(e',q)=q$, for all $q\in\Psi$. Then
	$K(M)^{X},(w,e')\vDash\neg p$.
	
	Let be given any attention action $Y=(\mathcal{A},\text{€},e)=((E,Q,Q^{*},pre),\text{€},e)$.
	Then $M^{Y},(w,e)\vDash p$: As by Def.~\ref{def: standard product updates},
	$M^{Y},(w,e)\vDash p$ iff $M,w\vDash p$. 
	
	Hence, $Y$ is not equivalent to $X$.
\end{proof}


\begin{thebibliography}{10}
	\providecommand{\url}[1]{\texttt{#1}}
	\providecommand{\urlprefix}{URL }
	\providecommand{\doi}[1]{https://doi.org/#1}
	
	\bibitem{Alechina2004}
	Alechina, N., Logan, B., Whitsey, M.: {A Complete and Decidable Logic for
		Resource-Bounded Agents}. In: {Proceedings of the Third International Joint
		Conference on Autonomous Agents and Multiagent Systems}. pp. 606--613 (2004)
	
	\bibitem{Aucher2013}
	Aucher, G., Bolander, T.: {Undecidability in Epistemic Planning}. In:
	Proceedings of the Twenty-Third International Joint Conference on Artificial
	Intelligence, {IJCAI-13}. pp. 27--33 (2013)
	
	\bibitem{Balbiani2019}
	Balbiani, P., Fern{\'a}ndez-Duque, D., Lorini, E.: {The Dynamics of Epistemic
		Attitudes in Resource-Bounded Agents}. {Studia Logica}  \textbf{107},
	457--488 (2019)
	
	\bibitem{BaltagBMS_1998}
	Baltag, A., Moss, L.S., Solecki, S.: {The Logic of Public Announcements, Common
		Knowledge, and Private Suspicions (extended abstract)}. In: Proceedings of
	the Seventh Conference on Theoretical Aspects of Rationality and Knowledge,
	{TARK-98}. pp. 43--56 (1998)
	
	\bibitem{BlueModalLogic2001}
	Blackburn, P., {\VAN{Rijke}{de}}~Rijke, M., Venema, Y.: {Modal Logic}.
	Cambridge University Press (2001)
	
	\bibitem{Bolander2017}
	Bolander, T.: {A gentle introduction to epistemic planning: The DEL approach}.
	In: Electronic Proceedings in Theoretical Computer Science. vol.~243, pp.
	1--22 (2017)
	
	\bibitem{BolanderBirkegaard2011}
	Bolander, T., Birkegaard, M.: {Epistemic planning for single- and multi-agent
		systems}. Journal of Applied Non-Classical Logics  \textbf{21}(1),  9--34
	(2011)
	
	\bibitem{Bolander2020}
	Bolander, T., Charrier, T., Pinchinat, S., Schwarzentruber, F.: {DEL-based
		epistemic planning: Decidability and complexity}. Artificial Intelligence
	\textbf{287},  103304 (2020)
	
	\bibitem{Bolander2016}
	Bolander, T., {\VAN{Ditmarsch}{van}}~Ditmarsch, H., Herzig, A., Lorini, E.,
	Pardo, P., Schwarzentruber, F.: {Announcements to Attentive Agents}. Journal
	of Logic, Language and Information  \textbf{25},  1--35 (2016)
	
	\bibitem{Bolander2015}
	Bolander, T., Jensen, M.H., Schwarzentruber, F.: {Complexity Results in
		Epistemic Planning}. In: Proceedings of the Twenty-Fourth International Joint
	Conference on Artificial Intelligence, {IJCAI-15} (2015)
	
	\bibitem{van2013}
	{\VAN{Ditmarsch}{van}}~Ditmarsch, H., Herzig, A., Lorini, E., Schwarzentruber,
	F.: {Listen to Me! Public Announcements to Agents that Pay Attention - or
		Not}. In: International Workshop on Logic, Rationality and Interaction. pp.
	96--109 (2013)
	
	\bibitem{Ditmarsch_Kooi_ontic}
	{\VAN{Ditmarsch}{van}}~Ditmarsch, H., Kooi, B.: {Semantic Results for Ontic and
		Epistemic Change}. In: Logic and the Foundations of Game and Decision Theory.
	pp. 87--117 (2008)
	
	\bibitem{Fagin_etal_1995}
	Fagin, R., Halpern, J.Y., Moses, Y., Vardi, M.Y.: {Reasoning About Knowledge}.
	The MIT Press (1995)
	
	\bibitem{hefti2015economics}
	Hefti, A., Heinke, S.: On the economics of superabundant information and scarce
	attention. {\OE}conomia. History, Methodology, Philosophy (5-1),  37--76
	(2015)
	
	\bibitem{Lorini2005}
	Lorini, E., Tummolini, L., Herzig, A.: {Establishing Mutual Beliefs by Joint
		Attention: Towards a Formal Model of Public Events}. In: Proceedings of the
	Annual Meeting of the Cognitive Science Society. pp. 1325--1330 (2005)
	
	\bibitem{sep-logic-epistemic}
	Rendsvig, R., Symons, J.: {Epistemic Logic}. In: Zalta, E.N. (ed.) The
	{Stanford} Encyclopedia of Philosophy. Metaphysics Research Lab, Stanford
	University, summer 2021 edn. (2021)
	
	\bibitem{Schipper2014}
	Schipper, B.C.: {Awareness}. In: {\VAN{Ditmarsch}{van}}~Ditmarsch, H., Halpern,
	J.Y., van~der Hoek, W., Kooi, B.P. (eds.) Handbook of Epistemic Logic (2014)
	
	\bibitem{simon1971}
	Simon, H.A.: {Designing organizations for an information-rich world}. In:
	Greenberger (ed.) Computers, communications, and the public interest, pp.
	37--72 (1971)
	
	\bibitem{Solaki2018}
	Smets, S., Solaki, A.: {The Effort of Reasoning: Modelling the Inference Steps
		of Boundedly Rational Agents}. In: Moss, L.S., de~Queiroz, R., Martinez, M.
	(eds.) Logic, Language, Information, and Computation. pp. 307--324 (2018)
	
\end{thebibliography}
\end{document}